\documentclass[sigconf, nonacm]{acmart} % For LaTeX2e
\usepackage{amsthm}

\newtheorem{proposition}{Proposition}
\newtheorem{corollary}{Corollary}
\newtheorem{lemma}{Lemma}
\newtheorem{remark}{Remark}
\newtheorem{definition}{Definition}
\newtheorem{example}{Example}
\newtheorem{conjecture}{Conjecture}

\usepackage{microtype} % Slightly tweak font spacing for aesthetics
\usepackage{babel}
\usepackage{thmtools, thm-restate}
\usepackage{todonotes} % !!!already included in clearR
\let\svtodo\todo
\renewcommand\todo[1]{\svtodo[inline]{#1}}
\usepackage{xcolor}
\usepackage{url}
\usepackage{amsmath} % essential for compilation, missing cause error
\usepackage{times}
\usepackage{hyperref}
\usepackage{cleveref}%after amsthm and hyperref
\crefname{theorem}{theorem}{theorems}
\crefname{corollary}{corollary}{corollary}
\Crefname{corollary}{corollary}{corollary}

\usepackage{graphicx} % Required for inserting images
\usepackage{subcaption} % Required for creating figures with multiple parts (subfigures)

\usepackage{algorithm}      % For the algorithm environment
\usepackage{algpseudocode}  % For algorithmic syntax

\newcommand{\ipg}{\textbf{ipg}} % inducing path graph
\newcommand{\upath}{\rho}
\newcommand{\induceP}{\rho_p}

\renewcommand{\citet}{\citep} % kdd seems not supporting \citet

\usepackage{hyperref}

\title[understanding solution space of causal discovery with unobserved confounding]{Addressing pitfalls in implicit unobserved confounding synthesis using explicit block hierarchical ancestral sampling
}

\author{Xudong Sun}
\affiliation{\department{Institute of AI for Health}\institution{Helmholtz Munich}\country{Germany}}
\authornote{Correspondance: \href{mailto:ac.xudong.sun@gmail.com}{ac.xudong.sun@gmail.com}}

\author{Alex Markham}
\affiliation{\department{Department of Mathematical Sciences}\institution{University of Copenhagen}\country{Denmark}}

\author{Pratik Misra}
\affiliation{\department{Department of Mathematics}\institution{Technical University of Munich}\country{Germany}}
\author{Carsten Marr}
\affiliation{\department{Institute of AI for Health}\institution{Helmholtz Munich}\country{Germany}}

\begin{document}
\sloppy
\begin{abstract}
	Unbiased data synthesis is crucial for evaluating causal discovery algorithms in the presence of unobserved confounding, given the scarcity of real-world datasets.
A common approach, {\em implicit parameterization}, encodes unobserved confounding by modifying the off-diagonal entries of the idiosyncratic covariance matrix while preserving positive definiteness.
Within this approach, we identify that state-of-the-art protocols have two distinct issues that hinder unbiased sampling from the complete space of causal models:
first, we give a detailed analysis of use of diagonally dominant constructions restricts the spectrum of partial correlation matrices; 
and second, the restriction of possible graphical structures when sampling bidirected edges, unnecessarily ruling out valid causal models.
To address these limitations, we propose an improved explicit modeling approach for unobserved confounding, leveraging block-hierarchical ancestral generation of ground truth causal graphs. Algorithms for converting the ground truth DAG into ancestral graph is provided so that the output of causal discovery algorithms could be compared with.
We draw connections between implicit and explicit parameterization, prove that our approach fully covers the space of causal models, including those generated by the implicit parameterization, thus enabling more robust evaluation of methods for causal discovery and inference.

%%% Local Variables:
%%% mode: LaTeX
%%% TeX-master: "main_kdd"
%%% End:

\end{abstract}

\maketitle

\section{Introduction}\label{sec:introduction}
%Causal discovery algorithms dealing with unobserved confounding~\citep{tashiro2014parcelingam,bhattacharya2021differentiable,wangdrton23} usually get evaluated on simple synthetic datasets plus some well known real world datasets, like the biology dataset~\textit{Sachs}~\citep{sachs2002bayesian}.

Real-world datasets~\citep{sachs2002bayesian,chevalley2022causalbench} are often only partially observed and high-dimensional. Their scarcity and selection bias prevent a high fidelity evaluation of causal discovery algorithms which calls for inclusive data synthesis.

Recent studies have increasingly scrutinized methodological issues using synthetic data derived from causal directed acyclic graphs (DAGs) without confounding.
For example, \citet{rios2021benchpress} provide a comprehensive collection of learning algorithms, real and synthetic datasets, as well as evaluation metrics.
\citet{reisach2021beware} identify a common bias in synthetic data--one that learning algorithms can exploit---and propose solutions to mitigate inflated performance evaluations.
Additionally, \citet{gamella2022characterization} present synthetic and semi-synthetic data generation tools that enhance the robustness of empirical evaluations for learning algorithms.

However, there has been considerably less attention to methodological challenges related to data generation when unobserved (latent) confounding is present.
Existing protocols typically consider 
confounders to be root variables or source nodes in the graph \citep[see][]{tashiro2014parcelingam,bhattacharya2021differentiable,wangdrton23}. Except for the major pitfalls we will unfold, these protocols also suffer from issues such as overly homogeneous graph generation (e.g., a uniform edge occurrence distribution) and fixed weight ranges.
%Furthermore, many protocols rely exclusively on zero-mean, mutually independent idiosyncratic variables and confounding \citep[see][]{tashiro2014parcelingam,bhattacharya2021differentiable,wangdrton23}, which leaves a large gap between the synthetic data and the diverse distributions observed in practice.
%Consequently, these limitations hinder a comprehensive evaluation of causal discovery algorithms under unobserved confounding.

In this paper, we take an initial step on best practices towards bridging this gap by analyzing common pitfalls in unobserved confounding data synthesis in implicit parameterization space, draw connections between implicit and explicit confounder formulation, and proposing a more general data synthesis protocol\footnote{\url{https://github.com/marrlab/causalspyne}} that incorporates heterogeneous graph and data distributions.%thereby moving synthetic data closer to real-world scenarios.

Our major contributions are as follows:

\begin{itemize}
  \item We analyze the solution space of causal discovery with unobserved confounding~\cite{drton2010parametrization} and identify a widespread limitation in the implicit unobserved confounding synthesis protocol based on diagonally dominant construction of the idiosyncratic covariance matrix.
	      This construction restricts the spectrum of the partial correlation matrix and fails to cover the entire space of possible distributions.

    \item Similarly, we point out a common restriction on the possible graphical structures that excludes valid causal models.

	\item To address these limitations, we propose an improved explicit confounder synthesis protocol with a block-hierarchical data generation scheme enabling the generation of heterogeneous graphs.
	      By selecting hidden (unobserved) confounders according to their topological order, we compute the ancestral graphs---rather than using directly sampled ancestral graphs---as the ground truth for evaluating causal discovery algorithms.

	\item We draw attention to the fact that unobserved confounders need not be restricted to root vertices.
	      In particular, when hidden confounders are children of observables. %todo(for example, when two unobserved confounders are confounded by observables), causal discovery algorithms exhibit different behaviors compared to when the hidden confounders are root vertices.

	\item Additionally, we establish the connection between the implicit and explicit confounder synthesis schemes, and we discuss the parameter constraint one could introduce in different ground truth DAGs in the explicit scheme.
        \item We offered detailed proof for alternative formulation of partial correlation coefficients using orthogonal projections~\citep{antti1983connection}.
        \item We elaborated the connection between ancestral graph and inducing path graph terminologies.
\end{itemize}

%%% Local Variables:
%%% mode: LaTeX
%%% TeX-master: "main_kdd"
%%% End:
                   % polished
\subsection{Paper structure}
In~\Cref{sec:notations}, we introduce the common notations used throughout this paper.
We present the basics of DAG encoding of CI (conditional independence) relationships in~\Cref{sec:dag_ci} and discuss the challenges posed by unobserved confounding in~\Cref{sec:dag_problem}.
This discussion naturally leads to concepts such as inducing path graphs, ancestral graphs and ancestral ADMGs (Acyclic Directed Mixed Graphs) in~\Cref{sec:ancestral_graph}.
Readers who are not familiar with these concepts are encouraged to review this section carefully.

In~\Cref{sec:sem}, we introduce the Structural Equation Model (SEM), presented with an implicit parameterization of unobserved confounding in~\Cref{sec:dgp_implicit_sem}.
We then discuss the state-of-the-art data generation protocol associated with this implicit formulation and address its inherent issues in~\Cref{sec:dgp_old_admg_protocol}.
The connection between explicit and implicit formulations is discussed in~\Cref{subsec:explicit_reformulate_implicit}.

We further examine the spectral limitations imposed by the diagonally dominant construction in~\Cref{sec:spectral_restrict_diag_dom}.

Finally, in~\Cref{sec:new_dgp}, we introduce our new block-hierarchical data synthesis and evaluation protocol, 
%present empirical evaluations in~\Cref{sec:experiment}, 
and conclude the paper in~\Cref{sec:conclusion}.

%%% Local Variables:
%%% mode: LaTeX
%%% TeX-master: "main_kdd"
%%% End:
         % polished
\section{Structural Equation Model with unobserved confounding}\label{sec:sem}
\subsection{Implicit Parameterization of Unobserved Confounding and the Joint \((W, \Omega)\) Space}
\label{sec:dgp_implicit_sem}

Implicit formulations of unobserved confounders are widely used in the literature~\citep{drton2010parametrization, bhattacharya2021differentiable, wangdrton23}; the main advantage of such approaches is that the number of parameters depends only on the number of observed variables~\citep{sullivant2023algebraic} and not on the number of unobserved variables.
In this section, we present a detailed structural equation model (SEM) formulation of implicit unobserved confounding.
We then provide a formal definition of implicit parameterization in~\Cref{def:implicit_unobserved_confounding} and introduce other essential basic concepts.

First, let \(L\) be a lower triangular matrix with zeros on the diagonal (thus it represents a given topological order), and let \(P\) be a permutation matrix.
Then the adjacency matrix \(W\) is constructed as
\begin{equation}
	W = P L P^T \label{eq:wplp}
\end{equation}
following \citet{mckay2003acyclic}.
In this formulation, the entry \(W_{(u,v)}\) corresponds to the edge from \(v\) to \(u\) (i.e., an edge into \(u\)), since the lower triangular structure of \(L\) ensures that if \(W_{(u,v)} \not=0\) then there is an arrow \(v \rightarrow u\).

The random vector $Y$ is defined by
\begin{equation}
	Y = W Y + \epsilon_Y, \label{eq:yeqwy}
\end{equation}
so that $Y$ can equivalently be expressed in terms of the idiosyncratic noise vector $\epsilon_Y$:
\begin{equation}
	Y = (I - W)^{-1} \epsilon_Y. \label{eq:yeqnoise}
\end{equation}
The covariance matrix of the idiosyncratic noise is given by
\begin{equation}
	\Omega := \mathbb{E}(\epsilon_Y \epsilon_Y^T), \label{eq:Omega_cov_idio}
\end{equation}
where a diagonal $\Omega$ corresponds to mutually independent idiosyncratic variables (i.e., the case of full observation or causal sufficiency). In contrast, nonzero symmetric off-diagonal entries of $\Omega$ encode unobserved confounding. Consequently, the covariance matrix of $Y$ is
\begin{equation}
  \Sigma := \mathbb{E}(Y Y^T) = {(I - W)}^{-1}\Omega {(I - W)}^{-T}, \label{eq:corr_y}
\end{equation}
which is a congruent transformation of $\Omega$~\citep{sullivant2023algebraic}.

\begin{remark}[Source Vertex]
	In \Cref{eq:wplp}, the lower triangular matrix $L$ guarantees the acyclicity of the graph by ensuring that its first row is entirely zero. Since permutation does not alter the fact that there is at least one zero row in $W$, this row corresponds to the root (or source) vertex (vertices) of the graph.
\end{remark}

\begin{definition}[Bidirected PD Cone~\citep{sullivant2023algebraic}]\label{def:pd_cone_bidirect}
	Define the bidirected positive definite (PD) cone as
	\begin{equation}
		\mathcal{C}_B = \{\Omega \mid \Omega \succ 0, \ \exists\, i \neq j \text{such that } \Omega_{i,j} = \Omega_{j,i} \neq 0\}, \label{eq:pd_cone_bidirect}
	\end{equation}
	i.e., the set of covariance matrices that have symmetric nonzero off-diagonal entries.
\end{definition}

\begin{definition}[Implicit Parameterization of Unobserved Confounding]\label{def:implicit_unobserved_confounding}
	By setting $\Omega \in \mathcal{C}_B$ in \Cref{eq:Omega_cov_idio} (see also \Cref{eq:pd_cone_bidirect}), we implicitly introduce confounding between variables $i$ and $j$ whenever $\Omega_{i,j} = \Omega_{j,i} \neq 0$.
\end{definition}

\begin{definition}[General Joint $(W, \Omega)$ Space]\label{def:joint_Omega_W_space}
	Define the joint space of the edge weight matrix $W$ and the idiosyncratic covariance $\Omega$ as
	\begin{equation}
		\mathcal{S}_B = \{ (W, \Omega) \mid \Omega \in \mathcal{C}_B, \text{and } W \text{represents a DAG} \}.
	\end{equation}
\end{definition}

\begin{definition}[Ancestral Restricted Joint $(W, \Omega)$ Space]\label{def:joint_Omega_W_space_admg}
	When sampling ancestral acyclic directed mixed graphs (ADMGs), it is often desirable to avoid almost-directed-cycles that can cause ambiguity in ancestral relationships. In particular, if variables $i$ and $j$ are confounded via an unobserved variable (i.e., $\Omega_{i,j} = \Omega_{j,i} \neq 0$), then for all $k$ we enforce
	$$[W^k]_{i,j} = [W^k]_{j,i} = 0,$$
	which guarantees that there is no ancestral relationship between $i$ and $j$. Thus, the corresponding restricted joint space is given by
	\begin{multline}
		\mathcal{A}_B = \Big\{ (W, \Omega) \,\Big|\,
		\Omega \in \mathcal{C}_B \text{ and } \forall\, i,j:\; \Omega_{i,j} = \Omega_{j,i} \neq 0\\
		\implies [W^{-1}]_{i,j} = [W^{-1}]_{j,i} = 0 \Big\}. \label{eq:pd_cone_bidirect_W_admg}
	\end{multline}
\end{definition}

\begin{proposition}\label{prop:admg_space_smaller_than_implicit_parameterization}
	It holds that $\mathcal{S}_B \setminus \mathcal{A}_B \neq \emptyset$, i.e., the ancestral restricted joint space $\mathcal{A}_B$ does not cover the entire joint space $\mathcal{S}_B$.
\end{proposition}

\begin{proof}
	Note that the general joint parameterization in $\mathcal{S}_B$ may include cases where for some $i$ and $j$ we have $\Omega_{i,j} \neq 0$ and also $[W^k]_{i,j} \neq 0$ (or $[W^k]_{j,i} \neq 0$) for some $k$. Such configurations do not satisfy the ancestral restriction imposed in $\mathcal{A}_B$, hence $\mathcal{S}_B \setminus \mathcal{A}_B \neq \emptyset$.
\end{proof}

\begin{remark}
	For elements in $\mathcal{S}_B \setminus \mathcal{A}_B$, a nonzero off-diagonal entry $\Omega_{i,j} \neq 0$ may correspond to a directed edge between $i$ and $j$ in the underlying ADMG rather than a purely bidirected edge.
\end{remark}
%%% Local Variables:
%%% mode: LaTeX
%%% TeX-master: "main_kdd"
%%% End:
      % polished
%In graphical models, the factorization of the probability density function is only dependent on the parents of a vertex. This allows us to obtain the covariances recursively using the trek rule. 
%In the linear Gaussian case, the class of O-distributions (distributions w.r.t.~observables) from the two data synthesis method are O-Markov indistinguishable, they can still differ in parameterization. To illustrate this, we start with the following definitions.
\begin{definition}[trek]
	A simple trek in DAG $D$ between vertex $u,v$ with top $k$: $trek(u, v;k|D)$ is an unordered pair of directed path ($\rho(k, u), \rho(k,v)$), such that $\rho(k,u)$ sinks at $u$, $\rho(k,v)$ sinks at $v$, $\rho(k,u)\cap \rho(k,v)=k$ and $k$ is the common source of both path which is called top, which can be a pair of confounded vertices $t_1,t_2$ or a single vertex $t_1=t_2$.
\end{definition}

\begin{definition}[trek monomial]\label{def:trek_monomial}
	\begin{equation}
		m_{trek(u,v;t_1,t_2|D)}=\Omega_{t_1,t_2}\prod_{s_1, s_2\in~trek(u,v;t_1,t_2|D)} W_{s_1,s_2}
	\end{equation}
\end{definition}
\begin{restatable}{proposition}{LabelRestatableTrekRule}[trek rule]\label{prop:trek_rule}
	%\begin{proposition}[trek rule]
	Let $T(u,v|D)$ represent the set of treks between $u,v$.
	\begin{align}
		\Sigma       & :={(I-W)}^{-T}\Omega{(I-W)}^{-1}                      \\
		\Sigma_{u,v} & = \sum_{trek(u,v;k|D) \in T(u,v|D)} m_{trek(u,v;k|D)}
	\end{align}
	%\end{proposition}
\end{restatable}
\begin{proof}
	See~\Cref{proof:trek_rule}.
\end{proof}
                    % polished
\subsection{Problems in the State-of-the-Art Unobserved Confounding Synthesis Protocol}\label{sec:dgp_old_admg_protocol}
%\todo{Check: Should we place this section here? I positioned it here since it is easy to refer to definitions from the previous section (e.g., $\Omega$, $\Sigma$, etc.)}

From the perspective of acyclic directed mixed graphs (ADMGs), the current state-of-the-art unobserved confounding synthesis approach generates (ancestral) ADMGs first and then simulates data based on the generated graphs.
%In \citet{bhattacharya2021differentiable}, additional restrictions are imposed so that the resulting ADMGs are both arid and bow-free.
We outline here some key aspects of the protocol and point out several limitations.

\paragraph{Ancestral ADMG Generation}

To generate an ancestral ADMG, the protocol first uniformly samples directed edges from a lower triangular matrix with a specified probability.
Then, bidirected edges are added only among pairs of observable vertices that do not share an ancestral relationship \citep{wangdrton23}.
This constraint, which is critical to prevent the formation of almost-directed cycles \citep{zhang2008causal}, can result in too few bidirected edges—thus, eliminating many potentially valid configurations from the ancestral ADMG space.

However, this constraint neglects the possibility that a pair of observables, say $u$ and $v$, may have an ancestral relationship (e.g., $u\in \operatorname{an}(v)$) while still being confounded by another variable $\xi$.
In such cases, the ancestral graph does not feature a bidirected edge between $u$ and $v$, although confounding is present.
In general, ancestral graphs need not contain bidirected edges at all (see~\Cref{fig:dag2ancestra_vs_fci_intermediate}).
Although in some sampling realizations $u$ and $v$ could be confounded by $\xi$, the constraint imposed on the conditional distribution $P(\Omega \mid W)$---with $W$ already dictating the ancestral relation between $u$ and $v$---precludes many valid ancestral ADMGs.
This issue is stated formally in~\Cref{prop:admg_space_smaller_than_implicit_parameterization}.

\paragraph{Positive Semidefiniteness of $\Omega$}

To ensure that the covariance matrix $\Omega$ is positive semidefinite (p.s.d.), the protocol enforces a condition such as~\citep{wangdrton23}:
\[
	\Omega_{ii} = 1 + \sum_{j \neq i} \lvert \Omega_{i,j} \rvert,
\]
where $\Omega_{i,j} = \Omega_{j,i}\neq 0$ represent the off-diagonal entries that model the confounding between vertex $i$ and $j$.
As shown in~\Cref{sec:spectral_restrict_diag_dom}, such constructions considerably restrict the spectrum of $\Omega$ and do not cover the full p.s.d.\ cone described in~\Cref{def:pd_cone_bidirect} for observable covariances.
%todo: either repeatedly shrinks the scale of the off-diagonal entries by a factor of $0.97$ until the minimum eigenvalue exceeds a threshold (0.01 in~\citep{wangdrton23}), or 

\paragraph{Uniform Edge Sampling}

A further issue involves the use of uniform edge sampling, which yields homogeneous graph structures.
Such a uniform sampling scheme is unlikely to generate heterogeneous graphs such as those studied in~\citep{tashiro2014parcelingam, wangdrton23}.
Additionally, it is common to restrict the edge weights to a fixed interval (e.g., $W_{i,j}\in[-0.5,0.5]$ as in~\citep{wangdrton23}).
As we demonstrate in~\Cref{thm:spectrum_par_corr}, such constraints contribute to a narrowing of the spectrum of the resulting partial correlation matrix.

In summary, these issues suggest that the current protocol for synthesizing unobserved confounding exhibits significant limitations, particularly in its restriction of feasible ADMG structures and in its treatment of the covariance matrix $\Omega$.
Future work should consider refining these aspects to better capture the diversity encountered in real-world networks.

%%% Local Variables:
%%% mode: LaTeX
%%% TeX-master: "main_kdd"
%%% End:
 % polished
\subsection{Explicit Reformulation of the Implicit Parameterization of Unobserved Confounders in SEM}\label{subsec:explicit_reformulate_implicit}
% Under partial observation, the implicit method considers unobserved confounders $\xi\in U$ as source nodes by filling the off-diagonal entries of the adjacency matrix $\Omega$ in~\Cref{eq:Omega_cov_idio} with nonzero values~\citep{sullivant2023algebraic}, without explicitly modeling the effects of unobserved confounders on observables.
In constrast to the implicit parameterization of unobserved confounders, we can also model the unobserved confounding explicitly in a SEM (structral equation model) associated with a DAG. In this seciton, we draw connections of the two via providing an explicit reformulation of the unobserved confounding model.

Let the system of variables $Y^{'}$ be composed of an observable subset $Y_O=Y$, as in~\Cref{eq:root_confounder_data_gen}, whose indices form the set $J_O$, and an unobserved (or unmeasured) subset $Y_U=\xi$, indexed by the set $J_U$.

\begin{restatable}{lemma}{LabelRestatableExplicitPD}\label{prop:explicit_always_psd}
%\begin{lemma}\label{prop:explicit_always_psd}
	Let $\xi \in \mathbb{R}^{|J_U|}$ be a random vector with mutually independent components, let $\epsilon_O \in \mathbb{R}^{|J_O|}$ be another random vector with mutually independent components, and assume that $\xi \perp \epsilon_O$.
	Further, let $\Lambda \in \mathbb{R}^{|J_O|\times |J_U|}$ be a matrix.
	Then, the matrix
	\begin{equation}\label{eq:explicit_always_psd}
		\Omega = \Lambda\,\mathbb{E}(\xi\xi^T)\,\Lambda^T + \mathbb{E}(\epsilon_O\,\epsilon_O^T)
	\end{equation}
	is positive definite.
%\end{lemma}
\end{restatable}
\begin{proof}
See~\Cref{proof_prop:explicit_always_psd}.
\end{proof}

\begin{remark}
  Since identical diagonal matrices perturbs the eigenvalues of a matrix by the diagonal value itselfwithout changing the eigenvectors, due to continuity, the form
	\[
		\Lambda\,\mathbb{E}(\xi\xi^T)\,\Lambda^T+\mathbb{E}(\epsilon_O\,\epsilon_O^T)
	\]
	spans the space of all symmetric positive definite matrices of size $|J_O|\times |J_O|$.
\end{remark}

\begin{restatable}{proposition}{LabelRestatableExplicitDGP}[Explicit Reformulation of Implicit Parameterization]\label{prop:explicit_dgp}
%\begin{proposition}[Explicit Reformulation of Implicit Parameterization]\label{prop:explicit_dgp}
  Suppose the observable variables $Y_O$ are generated according to an inter-observable causal structure defined by the adjacency matrix $W_o=W$ where $W$ in~\Cref{eq:corr_y}, while the effect of unobserved confounding is captured by the matrix $\Lambda$.
	In addition, assume there is an idiosyncratic error vector $\epsilon_O$, which is mutually independent of $\xi$ (see~\citep{shimizu2006linear, hoyer2008estimation}).
	Then the data-generating process is given by
	\begin{align}
		Y_O = \;               & W_oY_O + \Lambda\,\xi + \epsilon_O, \label{eq:oeqbolambdaxi}                                                                                             \\[1mm]
		\forall\, u \in J_U,\; & \exists\, v_1,\, v_2 \in J_O,\, v_1 \neq v_2 \quad\text{such that}\quad \Lambda_{v_1,u}\neq 0,\, \Lambda_{v_2,u}\neq 0, \label{eq:constraint_mat_Lambda}
	\end{align}
	where condition~\eqref{eq:constraint_mat_Lambda} requires every unobserved confounder $\xi_u$ to affect at least two observables.
	Then, equations~\eqref{eq:oeqbolambdaxi} and \eqref{eq:constraint_mat_Lambda} can be equivalently reformulated via the implicit formulation in \Cref{eq:wplp,eq:corr_y}, with the covariance matrix $\Omega$, defined in \Cref{eq:Omega_cov_idio}, having nonzero off-diagonal entries.
%\end{proposition}
\end{restatable}
\begin{proof}
  See~\Cref{proof:explicit_dgp}.
\end{proof}
\begin{proposition}[Block Adjacency Matrix Reformulation for Causally Sufficient DAG]\label{prop:implicit_confound_dag_view}
	The implicit formulation of unobserved confounding (given by \Cref{eq:wplp,eq:corr_y} with a non-diagonal $\Omega$ as in \Cref{eq:Omega_cov_idio}) can be interpreted as representing a causally sufficient DAG.
	In this DAG, the unobserved confounder $\xi$ appears as a root node and the observed variables $Y_O$ are the leaf nodes, with the off-diagonal entries of $\Omega$ corresponding to the edges from $\xi$ to $Y_O$, the corresponding SEM is:
		\begin{equation}\label{eq:root_confounder_data_gen}
		Y' :=
		\begin{pmatrix}
			Y_O \\[1mm] \xi
		\end{pmatrix}
		=
		\begin{bmatrix}
			W_o & \Lambda \\[1mm] 0 & 0
		\end{bmatrix}
		\begin{pmatrix}
			Y_O \\[1mm] \xi
		\end{pmatrix}
		+
		\begin{pmatrix}
			\epsilon_O \\[1mm] \xi
		\end{pmatrix}
                = W'\, Y' + \epsilon'
	\end{equation}
	where $W'$ is the global adjacency matrix \begin{equation}
          W^{'}= \begin{bmatrix}
			W_o & \Lambda \\[1mm] 0 & 0
		\end{bmatrix}
	\end{equation}
\end{proposition}

\begin{proof}
Rewrite equation~\eqref{eq:oeqbolambdaxi} by making the dimensions explicit (using brackets $[\cdot]$ to denote matrices/vectors where needed):
%	\[
%		Y_O = W_o\,Y_O + \Lambda\,\xi + \epsilon_O.
%	\]
	Under condition~\eqref{eq:constraint_mat_Lambda}, every confounder $\xi$ acts on at least two different observables, and hence when reformulated together the variables can be represented in the block form.
\end{proof}

\begin{restatable}{corollary}{LabelRestatableCoroExplicitDAG}\label{coro:explicit_dag}
%\begin{corollary}
	The following statements are equivalent:
	\begin{itemize}
		\item The explicit structural equation model exhibits a zero block for $\xi$ in \Cref{eq:root_confounder_data_gen}.
		\item The unobserved confounders are source vertices (i.e., they have no parents).
		\item The explicit DAG corresponds to a bipartite graph with one partition comprising $\xi$ and the other comprising $Y_O$.
	\end{itemize}
%\end{corollary}
\end{restatable}
\begin{proof}
See~\Cref{proof:coro_explicit_dag}.
\end{proof}
\begin{remark}
	When the true DAG is star-shaped with only a single unobserved confounder $\xi$, the matrix $\Lambda$ reduces to a \emph{column} vector.
\end{remark}
We are also interested in if there are alternative adjacency matrices other than $W^{'}$ in~\Cref{eq:root_confounder_data_gen} that can result in the same observable covariance matrix $\mathbb{E}(Y_O{Y_O}^{T})$ via similarity transform to original adjacency matrix $W^{'}$, 
which correspond to congruent transformation of the covariance matrix $\Sigma^{'}=\mathbb{E}(Y^{'}{Y^{'}}^{T})$, 
where $Y^{'}$ is defined in~\Cref{eq:root_confounder_data_gen}, and we have~\Cref{coro:alg_eq_invariance_sigma_transform,thm:w_transform}.
\begin{restatable}{proposition}{LabelRestatablePropAlgEqInvarianceSigmaTransform}\label{coro:alg_eq_invariance_sigma_transform}
%\begin{proposition}\label{coro:alg_eq_invariance_sigma_transform}
Let $\Sigma^{'}$ be an arbitrary covariance matrix with the following block structure:
  \begin{equation}
    \Sigma^{'}=\begin{bmatrix} 
      &\Sigma_{11} &\Sigma_{12}\\
      &\Sigma_{21} &\Sigma_{22}
    \end{bmatrix}\label{eq:sigma_block}
  \end{equation}

Let the congruent transformation matrix be 
\begin{equation}
Q=\begin{bmatrix}&Q_{11} &Q_{12}\\&Q_{21} &Q_{22}
\end{bmatrix}
\end{equation}

Any congruent transformation matrix $Q$ to $\Sigma^{'}$ that keeps the top-left block invariant, is a solution to the matrix equation %including $Q$ in~\Cref{eq:congruent_mat_schur_complement}, 
\begin{align}
&{\Sigma}_{11}\nonumber\\=&  Q_{11}^T\Sigma_{11}Q_{11}+Q_{21}^T\Sigma_{21}Q_{11} + Q_{11}^T\Sigma_{12}Q_{21}+Q_{21}^T\Sigma_{22}Q_{21}
\end{align}
%\end{proposition}
\end{restatable}

\begin{restatable}{theorem}{LabelRestatableThmTransformW}\label{thm:w_transform}
Using notation from~\Cref{eq:root_confounder_data_gen},
%\begin{theorem}
let
\begin{equation}
Q=\begin{bmatrix}&I &\Sigma_{11}^{-1}\Sigma_{12}\\&0 &I\end{bmatrix}
\end{equation}

\begin{align}
  W^{''}_1&=
\begin{bmatrix}
W_o & \Lambda - W_o \Sigma_{11}^{-1} \Sigma_{12} \\
\Sigma_{21}\Sigma_{11}^{-1} W_o & \Sigma_{21}\Sigma_{11}^{-1} (\Lambda - W_o \Sigma_{11}^{-1} \Sigma_{12})
\end{bmatrix}
\end{align}
where 
\begin{align}
\Sigma_{11}&:=
\mathbb{E}(Y_O{Y_O}^T)\nonumber\\
&=\mathbb{E}{(I-W_0)}^{-1}(\epsilon_O+\Lambda \xi)(\epsilon_O^T+ \xi^T\Lambda^T) {(I-W_0)}^{-T}\\
             &={(I-W_0)}^{-1}(I+\Lambda\Lambda^T){(I-W_0)}^{-T}
\end{align}
and
\begin{align}
\Sigma_{12}:=\mathbb{E}(Y_O \xi^T)&=\mathbb{E}{(I-W_0)}^{-1}(\epsilon_O+\Lambda \xi)\xi^T\\
           &={(I-W_0)}^{-1}\Lambda
\end{align}
and
\begin{equation}
\Sigma_{22}:=\mathbb{E}(\xi\xi^T)
\end{equation}

The adjacency matrix $W_1^{''}$ with idiosyncratic variable 
\begin{align}
  \epsilon^{''}=\begin{pmatrix}
     & \epsilon_O\\
     & \Sigma_{11}^{-1}\Sigma_{21}\epsilon_O +\xi
  \end{pmatrix}
\end{align}
correspond to observed data in the first block of $Y^{''}$:
\begin{align}
  Y^{''}=\begin{pmatrix}
     & Y_O\\
     & \Sigma_{11}^{-1}\Sigma_{21}Y_O +\xi
  \end{pmatrix}
\end{align}
in terms of invariant $\Sigma_{11}$ in
\begin{align}
  \Sigma^{''} &:= \mathbb{E}(Y^{''}{Y^{''}}^T) \\%= {(I - A)}^{-1}\Omega^{''}{(I - A)}^{-T}\\
&=\begin{bmatrix}
      &\Sigma_{11} &\Sigma^{''}_{12}\\
      &\Sigma^{''}_{21} & \Sigma^{''}_{22}\end{bmatrix} %[\Sigma / \Sigma_{11}]
\end{align}
%where $[\Sigma / \Sigma_{11}]$ is the Schur complement.
%\begin{equation}
%[\Sigma / \Sigma_{11}]= \Sigma_{22}-\Sigma_{21}\Sigma_{11}^{-1}\Sigma_{12}
%\end{equation}

The adjacency matrix $W_1^{''}$ results from the similarity transformation with matrix $Q^T$ (instead of $Q$) to the adjacency matrix $W^{'}$ of a DAG, where $W^{'}$ is the adjacency matrix in~\Cref{eq:root_confounder_data_gen}, 
$Q=\begin{bmatrix}&I &\Sigma_{11}^{-1}\Sigma_{12}\\&0 &I\end{bmatrix}$
%(\text{from}~\Cref{eq:congruent_mat_schur_complement})$
%\end{theorem}

Similarly, the adjacency matrix $W_2^{''}$ 
\begin{align}
W^{''}_2
=\begin{bmatrix}
W_o + \Lambda \Sigma_{21} \Sigma_{11}^{-1} & \Lambda \\
0 & 0
\end{bmatrix}\end{align}
with idiosyncratic variable
\begin{align}
  \epsilon^{''}=\begin{pmatrix}
     & \epsilon_O\\
     & \xi-\Sigma_{11}^{-1}\Sigma_{21}\epsilon_O 
  \end{pmatrix}
\end{align}
and observed data in the first block of $Y^{''}$:
\begin{align}
  Y^{''}=\begin{pmatrix}
     & Y_O\\
     & \xi -\Sigma_{21}\Sigma_{11}^{-T}Y_O
  \end{pmatrix}
\end{align}
correspond to the invariant $\Sigma_{11}$ in block diagonal matrix $\Sigma^{''}$ where \begin{equation}
  \Sigma^{''}=
\begin{bmatrix}
      &\Sigma_{11} &0\\
      &0 &[\Sigma / \Sigma_{11}] \end{bmatrix} =
      Q^{-T}\Sigma^{'} Q^{-1}\label{eq:sigma_schur_com_transform}
\end{equation}
and
  \begin{equation}
    \Sigma^{'}=\begin{bmatrix} 
      &\Sigma_{11} &\Sigma_{12}\\
      &\Sigma_{21} &\Sigma_{22}
    \end{bmatrix}\end{equation}

The adjacency matrix $W_2^{''}$ results from the similarity transformation with matrix $Q^{-T}$ (instead of $Q^T$) to the adjacency matrix $W^{'}$ of a DAG, where $W^{'}$ is the adjacency matrix in~\Cref{eq:root_confounder_data_gen}.
\end{restatable}

\begin{proof}
  See~\Cref{proof:thm_w_transform}.
\end{proof}

%%% Local Variables:
%%% mode: LaTeX
%%% TeX-master: "main_kdd"
%%% End:
             % polished
\section{Spectral Restrictions and the Coverage Problem Due to Diagonally Dominant Construction of $\Omega$}\label{sec:spectral_restrict_diag_dom}
\subsection{Construction of a Diagonally Dominant Covariance Matrix $\Omega$ for Positive Definiteness}

Recall from \Cref{prop:explicit_always_psd} that our explicit construction of the unobserved confounding matrix guarantees the positive definiteness of $\Omega$ in its corresponding implicit parameterization regardless of the number of confounders $|J_U|$.
In contrast, directly constructing a non-diagonal $\Omega$ in the implicit parameterization (e.g., by symmetrically filling off-diagonal elements in~\Cref{eq:Omega_cov_idio}) does not automatically yield a positive semidefinite (p.s.d.) matrix.
%For instance,
%\[
%	\begin{bmatrix}
%		1 & 2 & 2 \\[3mm]
%		2 & 1 & 2 \\[3mm]
%		2 & 2 & 1
%	\end{bmatrix}
%\]
%is not positive semidefinite.
%

A widely used technique in implicit confounder synthesis~\citep{drton2009computing, bhattacharya2021differentiable, wangdrton23} is to enforce positive definiteness by constructing $\Omega$ as a diagonally dominant matrix (see \Cref{def:diag_dom}).
(Note: \citep{tashiro2014parcelingam} forces the variance of $\epsilon_O$ to be half that of $Y$; it is an interesting question whether this construction is also diagonally dominant.) Our corresponding explicit formulation is given in \Cref{coro:omega_psd}, which is the data synthesis protocol employed in \citet{wangdrton23}.

\begin{restatable}[Special case of \Cref{prop:explicit_always_psd}]{corollary}{LabelRestableOnePlusPD}\label{coro:omega_psd}
%\begin{corollary}[Special case of \Cref{prop:explicit_always_psd}]\label{coro:omega_psd}
	A covariance matrix $\Omega$ constructed via
	\[
		\Omega_{i,i} = 1 + \sum_{j \neq i} |\Omega_{i,j}|
	\]
	is positive definite.
%\end{corollary}
      \end{restatable}
      \begin{proof}
See~\Cref{proof:1plusexplicit_always_psd}.
      \end{proof}
\begin{definition}[Diagonally Dominant~\citep{hillar2012inverses}]\label{def:diag_dom}
	A square matrix $\Omega$ of size $|V|$ is said to be diagonally dominant if
	\[
		\Delta_i(\Omega) = |\Omega_{i,i}| - \sum_{j \neq i} |\Omega_{i,j}| \ge 0, \quad \forall i.
	\]
\end{definition}

\begin{remark}
	The special construction of the $\Omega$ matrix according to \Cref{coro:omega_psd} imposes additional constraints (i.e., that $\Omega$ is diagonally dominant) and therefore only covers a subspace of the p.s.d.\ cone of covariance matrices defined in \Cref{def:pd_cone_bidirect} for the observables.
	There exist other constructions of $\Omega$ that ensure positive semidefiniteness.
%	For example,
%	\[
%		\begin{pmatrix}
%			0.1 & 0.2 \\[3mm]
%			0.2 & 10
%		\end{pmatrix}
%	\]
%	is p.s.d.~\citep{counter_ex_psd_not_diag_dom} but does not satisfy the diagonal dominance condition imposed by \Cref{coro:omega_psd}.
\end{remark}

\begin{definition}\label{def:corr_matrix_epsilon}
	Let $\Omega$ denote the covariance matrix of $\epsilon$ as given in~\Cref{eq:Omega_cov_idio}.
	The correlation matrix of $\epsilon$, denoted by $R_{\epsilon}$, is defined as
	\[
		R_{\epsilon} := \sqrt{[{\rm diag}(\Omega)]^{-1}}\;\Omega\;\sqrt{[{\rm diag}(\Omega)]^{-1}}.
	\]
	Similarly, let $R_{Y}$ denote the correlation matrix of $Y$, given by
	\[
		R_{Y} := \sqrt{[{\rm diag}(\Sigma)]^{-1}}\;\Sigma\;\sqrt{[{\rm diag}(\Sigma)]^{-1}}.
	\]
\end{definition}

\begin{proposition}\label{prop:corr_mat_diag_dom_due2_cov_diag_dom}
	If $\Omega$ is diagonally dominant and has identical diagonal entries, then the correlation matrix $R_{\epsilon}$ defined in \Cref{def:corr_matrix_epsilon} is also diagonally dominant.
\end{proposition}

\begin{proof}
	Since the diagonal entries of $R_{\epsilon}$ are given by
	\[
		R_{\epsilon_{i,i}} = \frac{\Omega_{i,i}}{\Omega_{i,i}} = 1,
	\]
	the corresponding off-diagonal entries are
	\[
		R_{\epsilon_{i,j}} = \frac{\Omega_{i,j}}{\sqrt{\Omega_{i,i}\Omega_{j,j}}}.
	\]
	Because $\Omega$ is diagonally dominant (i.e., $\Omega_{i,i} > \sum_{j\neq i}|\Omega_{i,j}|$) and all diagonal elements are identical, the scaling factors $\sqrt{\Omega_{i,i}}$ are constant across rows.
	Therefore, the diagonal dominance property transfers directly to $R_{\epsilon}$ since the relative magnitudes of the off-diagonals are preserved.
	Note also that $R_{\epsilon}$ has the same support (nonzero pattern) as $\Omega$.
\end{proof}

%%% Local Variables:
%%% mode: LaTeX
%%% TeX-master: "main_kdd"
%%% End:
  % polish this input too!

\subsection{Spectrum Restrictions due to a Diagonally Dominant $\Omega$ Matrix}

\begin{lemma}[Spectral Radius Bound for a Diagonally Dominant Matrix $\Omega$]\label{lemma:spectrum_radius_omega_diag_dom}
	For a diagonally dominant matrix $\Omega$, the spectral radius $\rho(\Omega)$ satisfies
	\begin{equation}\label{eq:spectrum_radius_diag_dom_diag_ele}
		0 < \rho(\Omega) \le 2\max_{i} |\Omega_{i,i}|.
	\end{equation}
\end{lemma}

\begin{proof}
	Inequality~\eqref{eq:spectrum_radius_diag_dom_diag_ele} follows directly from the Gershgorin circle theorem~\cite[Theorem 6.1.1]{horn2012matrix}.
	According to the definition of diagonal dominance (see~\Cref{def:diag_dom}), the Gershgorin disk centered at $\Omega_{ii} > 0$ with radius
	$$\sum_{j\neq i} |\Omega_{ij}| = \Omega_{ii} - \Delta_i(\Omega)$$
	can be bounded by a disk centered at $\Omega_{ii}$ with radius $\Omega_{ii}$.
	This observation yields the stated bound on the spectral radius.
\end{proof}

\begin{corollary}[Spectral Radius Bound for $R_{\epsilon}$ Under Diagonal Dominance]\label{coro:spectrum_radius_diag_dom_corr}
	When the correlation matrix $R_{\epsilon}$ is diagonally dominant, it holds that
	\begin{equation}\label{eq:spectrum_radius_diag_dom_corr}
		0 < \rho(R_{\epsilon}) \le 2.
	\end{equation}
\end{corollary}

\begin{proof}
	By \Cref{lemma:spectrum_radius_omega_diag_dom}, if $R_{\epsilon}$ is diagonally dominant then
	$$0 < \rho(R_{\epsilon}) \le 2\max_{i} \bigl|[R_{\epsilon}]_{i,i}\bigr|.$$
	Since $R_{\epsilon}$ is a correlation matrix, we have $[R_{\epsilon}]_{i,i}=1$ for all $i$, which implies the desired bound.
\end{proof}

\begin{remark}[Impact of Diagonal Dominance on the Coverage of the Correlation Matrix Set]\label{coro:spectrum_radius_diag_dom_corr_cover}
	For a general correlation matrix $R$, one has
	\begin{equation}\label{eq:spectrum_corr_cover}
		0 \le \rho(R) \le \sum_i\lambda_i = |V|,\quad
	\end{equation}
	where $|V|$ denotes the number of variables \citep{holmes1991random}.
	Thus, by imposing an upper bound on the spectral radius as in~\eqref{eq:spectrum_radius_diag_dom_corr} (namely, $\rho(R_{\epsilon})\le2$), the set of allowable correlation matrices is significantly reduced compared to the unconstrained set with the upper bound $|V|$.
	% The phrase "reduces the cover" can be interpreted as decreasing the class/size of possible correlation matrices.
\end{remark}

\subsection{Effect on Partial Correlation Matrix}
%The non-zero entries of $\Omega$ already encode the “position” of the confounding.
%Even though the diagonally dominant structure of $\Omega$ may not cover the full cone $\mathcal{C}_B$ in~\Cref{def:pd_cone_bidirect}, one might ask whether this construction affects the outcome of a causal discovery algorithm.
%\todo{(Can we claim this?) We show in~\Cref{thm:spectrum_par_corr} and~\Cref{remark:restricted_spectrum_partial_corr} that this indeed affects the CI test statistic distribution.}

\begin{restatable}{lemma}{LabelRestatableOmegaInvSpectrum}[Spectral radius of $\Omega^{-1}$ due to diagonally dominant $\Omega$]\label{lemma:inv_omega_diag_dom_spectral_radius}
%\begin{lemma}[Spectral radius of $\Omega^{-1}$ due to diagonally dominant $\Omega$]\label{lemma:inv_omega_diag_dom_spectral_radius}
	In general,
	\[
		\lambda(\Omega^{-1}) < \max_{1\le i\le |V|} \frac{1}{\Delta_i(\Omega)}.
	\]
	For the special construction of~\Cref{coro:omega_psd} with $\Delta_i(\Omega) > 1$, we have
	\[
		\lambda(\Omega^{-1}) < 1.
	\]
%\end{lemma}
\end{restatable}
\begin{proof}
  See \Cref{proof:inv_omega_diag_dom_spectral_radius}.
\end{proof}
\begin{remark}[General relationship between matrix entry interval and spectral radius]\label{lemma:matrix_entry_interval}
	\Cref{eq:spectrum_radius_diag_dom_interval} can be found in Theorem 7 of~\citep{zhan2005extremal}.
	This result is stated for a general real symmetric matrix with entries in a symmetric interval $[-a, a]$.
	Here, replacing $a$ with $\max \bigl(|\Omega^{-1}_{i,j}|\bigr)$, we have:
	\begin{equation}
		s(\Omega^{-1}) := \lambda_{\max}(\Omega^{-1}) - \lambda_{\min}(\Omega^{-1}) \le \sqrt{2}\max\Bigl(|\Omega^{-1}_{i,j}|\Bigr)|V|.\label{eq:spectrum_radius_diag_dom_interval}
	\end{equation}
	In general, a small spectral radius of $\Omega^{-1}$ (as in~\Cref{lemma:inv_omega_diag_dom_spectral_radius}) implies that $\max\bigl(|\Omega^{-1}_{i,j}|\bigr)$ is small, especially as the network size $|V|$ grows.
\end{remark}

\begin{remark}\label{remark:sigma_pos_omega_inv_le2_hillar}
	If $\Omega$ has positive entries (that is, for diagonally dominant, nonnegative covariance matrices $\Omega$) then its inverse (the precision matrix) satisfies the property
	\begin{align}
		\|\Omega^{-1}\|_{\infty} \le \frac{3|V|-4}{2\ell(|V|-2)(|V|-1)} & =: B(V),                        \\
		\ell                                                            & := \min \Omega_{i,j}, \nonumber
	\end{align}
	as shown in \citet{hillar2012inverses}.
	Then by~\Cref{lemma:inv_omega_diag_dom_spectral_radius},
	\[
		\lambda(\Omega^{-1}) < B(V).
	\]
	Thus, the spectrum of the precision matrix is bounded by $[0, B(V)]$, irrespective of $\Delta_i(\Omega)$.
\end{remark}

\begin{remark}[Congruent Transformation of a Diagonally Dominant Symmetric Matrix]\label{lemma:diag_dom_diag_congruent2non_diag_dom}
	In general, a congruent transformation does not preserve the diagonally dominant structure of a matrix.
	However, consider a matrix $\Sigma$ that is not diagonally dominant but has rapidly growing diagonal entries.
	Then, there exists a diagonal matrix $D=\operatorname{diag}(1,\rho,\ldots,\rho^{|V|-1})$ such that
	\[
		\Omega = D^T\Sigma D
	\]
	is diagonally dominant~\citep{barlow1990computing,lemma_connect_diag_dom2non}.
	Note that any congruent transformation preserves positive semidefiniteness.
\end{remark}

\begin{definition}[Partial Correlation]\label{def:partial_corr}
	Let the residual of the regression of variable $i$ on $V \setminus \{i,j\}$ be given by
	\begin{equation}
		\hat{\epsilon}_{i|i,j} = \min_{\beta} \left\| Y_i - \sum_{k\neq i,j}\beta_{i,k}Y_k \right\|,
	\end{equation}
	then the partial correlation between $i$ and $j$ (conditioning on the rest) is defined as
	\begin{equation}
		\rho_{i,j|\cdot} = \frac{\mathbb{E}(\hat{\epsilon}_{i|i,j}\hat{\epsilon}_{j|i,j})}{\sqrt{\mathbb{E}(\hat{\epsilon}_{i|i,j}^2)\mathbb{E}(\hat{\epsilon}_{j|i,j}^2)}}.
		\label{eq:def_partial_corr}
	\end{equation}
\end{definition}

\begin{restatable}{lemma}{LabelRestatableLemmaAlterResidualParCorr}\label{lemma:par_corr_alter_residual}
%\begin{corollary}[alternative residual partial correlation]\label{coro:expression_partial_corr_matrix}

	Let the residual from the regression of variable $i$ on $V \setminus \{i,j\}$ be defined as
	\begin{equation}
		\hat{\epsilon}_{i|i,j} = \min_{\beta} \left\| Y_i - \sum_{k\neq i,j}\beta_{i,k}Y_k \right\|.
	\end{equation}
	Let the residual from the regression of variable $i$ on $V \setminus \{i\}$ be defined as
	\begin{equation}
		\hat{\epsilon}_{i|i} = \min_{\beta} \left\| Y_i - \sum_{k\neq i}\beta_{i,k}Y_k \right\|.
	\end{equation}
	Then the partial correlation can be alternatively expressed as
	\begin{equation}
          \rho_{i,j\mid \cdot}=\frac{\mathbb{E}(\hat{\epsilon}_{i|i,j}\hat{\epsilon}_{j|i,j})}{\sqrt{\mathbb{E}(\hat{\epsilon}_{i|i,j}^2)\mathbb{E}(\hat{\epsilon}_{j|i,j}^2)}}
 = -\frac{\mathbb{E}(\hat{\epsilon}_{i|i}\hat{\epsilon}_{j|j})}{\sqrt{\mathbb{E}(\hat{\epsilon}_{i|i}^2)\mathbb{E}(\hat{\epsilon}_{j|j}^2)}}.
		\label{eq:def_partial_corr_alt}
	\end{equation}
%\end{corollary}
\end{restatable}
\begin{proof}
See~\Cref{proof:par_corr_exclude_one_var}.
\end{proof}

\begin{restatable}{lemma}{LabelLemmaPartialCorrPrecision}\label{lemma:par_corr_precision_mat}
	Let
	\[
		\Sigma^{-1} = \Bigl[(I-W)^{-1}\Omega (I-W)^{-T}\Bigr]^{-1}
	\]
	be the precision matrix.
	Then, the partial correlation between variables $i$ and $j$ can be expressed as
	\begin{align}
                &\rho_{i,j|\cdot} \nonumber\\
                & = -\frac{[\Sigma^{-1}]_{i,j}}{\sqrt{[\Sigma^{-1}]_{i,i}\,[\Sigma^{-1}]_{j,j}}} \label{eq:par_corr_first}                                                                    \\
		                 & = -\frac{\Bigl[(I-W)^{-1}\Omega^{-1}(I-W)^{-T}\Bigr]_{i,j}}{\sqrt{\Bigl[(I-W)^{-1}\Omega^{-1}(I-W)^{-T}\Bigr]_{i,i}\,\Bigl[(I-W)^{-1}\Omega^{-1}(I-W)^{-T}\Bigr]_{j,j}}}\,.
		\label{eq:par_corr}
	\end{align}
	The partial correlation matrix can equivalently be written as
	\begin{equation}
		\tilde{R} = -\operatorname{diag}\Bigl(\sqrt{[\Sigma^{-1}]^{-1}_{ii}}\Bigr)\,\Sigma^{-1}\,\operatorname{diag}\Bigl(\sqrt{[\Sigma^{-1}]^{-1}_{ii}}\Bigr)
		\label{eq:expression_partial_corr_matrix}
	\end{equation}
	or, in terms of the correlation matrix,
	\begin{equation}
		\tilde{R} = -\operatorname{diag}\Bigl(\sqrt{[R^{-1}]^{-1}_{ii}}\Bigr)\,R^{-1}\,\operatorname{diag}\Bigl(\sqrt{[R^{-1}]^{-1}_{ii}}\Bigr)
		\label{eq:expression_partial_corr_matrix_corr_mat}.
	\end{equation}
\end{restatable}

\begin{proof}
	See \Cref{proof:par_corr_precision}.
\end{proof}

\begin{restatable}{lemma}{LabelLemmaInvSigmaSpectrum}\label{lemma:spectrum_inv_sigma_congruent}
	Although a congruent transformation does not generally preserve the diagonally dominant structure of $\Omega$ (see \Cref{lemma:diag_dom_diag_congruent2non_diag_dom}), the precision matrix $\Sigma^{-1}$ defined in \eqref{eq:corr_y} preserves the spectral structure of $\Omega^{-1}$ via a magnifying factor $\theta$ bounded by the spectral radius $\rho(W+W^T)$.
	Specifically,
	\begin{align}
		\rho(\Sigma^{-1}) & \le \theta\, \rho(\Omega^{-1}) = \theta \max_i \frac{1}{\Delta_i(\Omega)} \label{eq:spectrum_inv_sigma_bound} \\[1mm]
		\theta            & \le \rho(W+W^T) + 1\,.
	\end{align}
\end{restatable}

\begin{proof}
	See \Cref{proof:inv_Sigma_spectrum}.
\end{proof}

\begin{lemma}\label{lemma:spectrum_radius_un_graph}
	The spectral radius of $W+W^T$ is bounded by
	\[
		\rho(W+W^T) \le d_{\max}\, \max_{i,j} W_{i,j}\,,
	\]
	where $d_{\max}$ is the maximum degree.
\end{lemma}

\begin{proof}
	Normalize an eigenvector $v$ of $\bar{W} = W+W^T$ so that for some index $i^*$ we have $v_{i^*}=1$ and $|v_j|\le1$ for all $j$.
	Then, using
	\[
		\bar{W}v = \lambda(\bar{W})\,v,
	\]
	its $i^*$th component satisfies
	\[
		\sum_j \bar{W}_{i^*,j} v_j = \lambda(\bar{W})\,.
	\]
	Thus,
	\[
		|\lambda(\bar{W})| \le \sum_j |\bar{W}_{i^*,j}| \le d_{\max}\, \max_{j} |\bar{W}_{i^*,j}| \le d_{\max}\, \max_{i,j} |\bar{W}_{i,j}|,
	\]
	and since $\max_{i,j} |\bar{W}_{i,j}| = \max_{i,j} W_{i,j}$ the claim follows.
\end{proof}

\begin{remark}
	In many settings the adjacency matrix $W$ is sparse with small entries.
	For example, in \citet{wangdrton23}, $\max W_{i,j}=0.5$, and if the maximum degree is on the order of $|V|$, then $d_{\max}\,\max W_{i,j}<1$ with high probability.
\end{remark}

\begin{corollary}\label{coro:spectrum_radius_sigma_inv}
	The spectral radius of the precision matrix $\Sigma^{-1}$ satisfies
	\[
		\rho(\Sigma^{-1}) \le \Bigl(d_{\max}\,\max W_{i,j}+1\Bigr)
		\max_i \frac{1}{\Delta_i(\Omega)}\,.
	\]
\end{corollary}

\begin{proof}
	This follows directly from \Cref{lemma:spectrum_inv_sigma_congruent}, \Cref{lemma:spectrum_radius_un_graph} and \Cref{lemma:inv_omega_diag_dom_spectral_radius}.
\end{proof}

\begin{restatable}{lemma}{LabelRestatebleLemmaPrecisionDiagLb}\label{lemma:lb_precision_diag_bt_1_over_sigma_diag_vif}
	The diagonal elements of the precision matrix are lower bounded by the reciprocals of the corresponding diagonal entries of the covariance matrix, i.e.,
	\[
		\Sigma_{p,p}^{-1}\ge \frac{1}{\Sigma_{p,p}}\,.
	\]
	See also \citet[Lemma 6]{harris2013pc}.
\end{restatable}

\begin{proof}
	See \Cref{lemma:lower_bound_precision_mat_diag_vif_proof}.
\end{proof}

\begin{corollary}\label{corollary:variance_inflation_factor_corr}
	The diagonal elements of the inverse correlation matrix satisfy
	\[
		R_{p,p}^{-1}\ge 1\,.
	\]
\end{corollary}

\begin{proof}
	As $R_{p,p} = \Sigma_{p,p} = 1$ for a correlation matrix, the result follows immediately from \Cref{lemma:lb_precision_diag_bt_1_over_sigma_diag_vif}.
\end{proof}

\begin{restatable}{lemma}{LabelRestatableLemmaUbDiagSigmaOmega}\label{lemma:ub_diag_Sigma_omega}
	We have
	\[
		\Sigma_{i,i} \le 2\,\max_i |\Omega_{i,i}|\,.
	\]
\end{restatable}

\begin{proof}
	See \Cref{proof_lemma:ub_diag_Sigma_omega}.
\end{proof}

\begin{restatable}{theorem}{LabelRestatableTheoremParCorr}\label{thm:spectrum_par_corr}
	The spectral radius of the partial correlation matrix corresponding to the diagonally dominant construction is bounded by
	\[
		\rho(\tilde{R}) \le \Bigl(2\,\max_i |\Omega_{i,i}|\Bigr)
		\Bigl(d_{\max}\,\max W_{i,j}+1\Bigr)\max_i\frac{1}{\Delta_i(\Omega)}\,.
	\]
\end{restatable}

\begin{proof}
	See \Cref{proof_thm:spectrum_par_corr}.
\end{proof}

\begin{remark}[Effect of a Restricted Spectrum]\label{remark:restricted_spectrum_partial_corr}
	Let $\lambda(\cdot)$ denote an eigenvalue of a matrix.
	In general, the eigenvalues of a partial correlation matrix satisfy $0\le|\lambda(\tilde{R})|\le|V|$ \citep{artner2022shape}.
	According to \Cref{thm:spectrum_par_corr}, the diagonally dominant construction of $\Omega$ limits the spectrum of the partial correlation matrix:
	\[
		\rho(\tilde{R}) \le \Bigl(2\,\max_i |\Omega_{i,i}|\Bigr)
		\Bigl(d_{\max}\,\max W_{i,j}+1\Bigr)\max_i\frac{1}{\Delta_i(\Omega)} \ll |V|\,.
	\]
	In light of \Cref{lemma:matrix_entry_interval}, a restricted spectral radius generally leads to a restricted range of matrix entries, which in turn can constrain the distribution of conditional independence (CI) test statistics based on partial correlation.
	%\todo{Check: Can we formally claim a restricted CI test statistic distribution, and what are its implications?}
\end{remark}

%%% Local Variables:
%%% mode: LaTeX
%%% TeX-master: "main_kdd"
%%% End:

\section{The block-hierarchical explicit unobserved confounding synthesis and performance evaluation}\label{sec:new_dgp}

To overcome the limitations discussed in Section~\ref{sec:spectral_restrict_diag_dom}, we explicitly model unobserved confounders.
In our approach, we first generate all variables—both observed and unobserved—in an acyclic (DAG) fashion (i.e., the ground truth DAG), and then hide the selected unobserved variables from the causal discovery algorithm.
For performance evaluation, we transform the DAG into ancestral ADMGs (as described in Section~\ref{sec:dag2ancestral}) that serve as the ground truth for the causal discovery algorithm.

Traditional methods (e.g., generating large DAGs via Erdo-Renyi graphs) tend to produce homogeneous graphs.
To remedy this, we adopt a block-hierarchical approach (see Section~\ref{sec:hdgp}).

\subsection{Hierarchical data generation via ancestral sampling}\label{sec:hdgp}

\subsubsection{An example DAG generated hierarchically}
\begin{figure}[h!]
	\centering
	\includegraphics[scale=0.5]{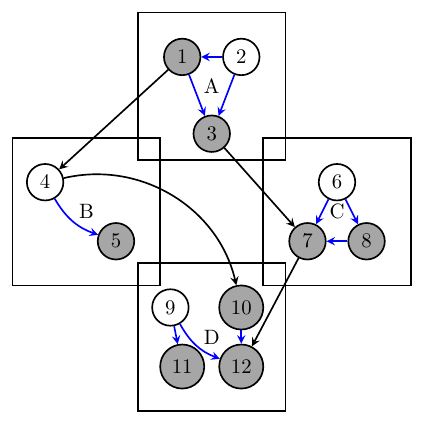}
	\caption{An example graph generated from our block-hierarchical data generation process with unobserved confounding where each block is represented as a rectangle what we call a macro node.
		In this example, unobserved node 4 is a child of observed node 1.
		Our algorithm allows control over joint dependencies across block structures.}
	\label{fig:hdgp}
\end{figure}

We first present the output and then describe its generation.
In Figure~\ref{fig:hdgp}, the following conventions are used:
\begin{itemize}
	\item A \emph{rectangle} (block) denotes a macro node.
	\item A \emph{circle} denotes a micro node.
	\item A \emph{macro-node DAG} is a DAG consisting only of macro nodes (e.g., the DAG on the set $\{A, B, C, D\}$ in Figure~\ref{fig:hdgp}; note that here the symbol $B$ denotes the macro node name rather than the adjacency matrix in Equation~\eqref{eq:root_confounder_data_gen}).
	\item A \emph{local micro-node DAG} is the DAG of micro nodes within a single macro node (e.g., the DAG on the set $\{1,2,3\}$ attached to macro node $A$).
	\item The \emph{global micro-node DAG} is defined over all micro nodes (e.g., with vertices $\{1,\ldots,12\}$).
	\item Shaded circles denote observed variables (which will be available to the causal discovery algorithm), while white circles denote unobserved variables.
	      In Figure~\ref{fig:hdgp}, nodes 2, 4, 6, and 9 are unobserved and serve as confounders.
\end{itemize}

\subsubsection{Protocol}
We now describe the algorithm and procedure to generate such a graph and the corresponding data:
\begin{enumerate}
	\item Generate the macro-node DAG (e.g., in Figure~\ref{fig:hdgp} the DAG with edges $A \rightarrow B$, $A \rightarrow C$, $C \rightarrow D$, and $B \rightarrow D$).
	      This can be done via Equation~\eqref{eq:wplp}.
	\item Populate each macro node with a local micro-node DAG (again via Equation~\eqref{eq:wplp}).
	\item Add inter macro-node connections by randomly selecting micro nodes from corresponding macro nodes.
	      Due to the macro-node DAG structure, this process does not introduce cycles.
	\item Enforce the expected number of confounders in the DAG (details in Section~\ref{sec:enforce_confounder}).
	\item Generate data via ancestral sampling (details in Section~\ref{sec:data_gen}).
	\item Choose which confounders to hide.
	      The corresponding columns in the data are removed, and the DAG (with the specified unobserved confounders) is transformed into an ancestral graph (see Section~\ref{sec:dag2ancestral}).
	      This graph serves as the ground truth for the causal discovery algorithm.
\end{enumerate}

The advantages of this protocol include:
\begin{itemize}
	\item Providing ground truth for marginal causal abstraction \citep{rubenstein2017causal,beckers2020approximate}.
	\item Enabling heterogeneous DAG properties (e.g., varying sparsity) by controlling higher-order joint dependencies via the inter macro-node connections.
	\item Reducing the topological sort time complexity from $|V|+|E_{intra}|+|E_{inter}|$ to $|V|+|E_{intra}|$, where $|E_{intra}|$ is the number of intra macro-node edges and $|E_{inter}|$ is the number of inter macro-node edges.
\end{itemize}

\subsubsection{Algorithm for ensuring an expected number of confounders}\label{sec:enforce_confounder}
To ensure sufficient confounding, we use the following procedure.
Given a topological order of the graph, we iterate over nodes starting from the lowest rank.
For a given node $v$, we examine if we can add an extra connection from some node $u$ (with $u \prec v$).
If $u$ is already a parent to more than one variable, then it ensures the existence of a confounder.
Regardless of success, we then choose the next variable (one rank higher than $u$) and continue until the expected number of confounders is reached.

\subsection{Data generation via ancestral sampling}\label{sec:data_gen}

\subsubsection{Ancestral data sampling}\label{sec:ancestral_sample}
For a given DAG, data are generated via ancestral sampling.
In the order specified by the topological sort, each variable is sampled conditional on its parents.
For source nodes, samples are drawn from a prescribed idiosyncratic distribution.
Compared to Equation~\eqref{eq:yeqnoise}, ancestral sampling offers the convenience of enforcing additional constraints, improved scalability in high-dimensional settings, reduced time and space complexity relative to matrix multiplication, and the flexibility to incorporate nonlinear parent–child relationships as well as non-Gaussian noise distributions.

\subsubsection{Alternative weight sampling}
Previous methods \citep{ogarrio2016hybrid,bhattacharya2021differentiable,wangdrton23} typically sample edge weights uniformly over a fixed interval.
Such homogeneous weight sampling may limit the diversity of edge strengths.
To address this, we propose an alternative based on Wishart sampling.
Specifically, we draw edge weights from a Wishart distribution with a positive definite scale matrix $S$ (by default, $S$ is taken as the identity matrix), and then randomly flip some of the edge weights to negative values.
The density is given by
\begin{align}
&f_{\mathbf{W}, \mathbf{S}} (\mathbf{W}, \mathbf{S}) \nonumber\\ =& \frac{1}{2^{d|V|/2} |\mathbf{S}|^{d/2} \Gamma_p\Bigl(\frac{d}{2}\Bigr)} |\mathbf{W}|^{(d-|V|-1)/2} \exp\Bigl(-\frac{1}{2}\operatorname{tr}\bigl(\mathbf{S}^{-1}\mathbf{W}\bigr)\Bigr),
	\label{eq:wishart}                                                                                                                                                                                                                             \\[1mm]
	d                                                   & = |V| + 1\,.
\end{align}
Here, the parameter $d$ and the choice of scale matrix $S$ control the diversity and the magnitude of the edge weights.

\subsection{Converting a DAG with Hidden Nodes to an Ancestral Graph}\label{sec:dag2ancestral}

\begin{figure}
	\centering
	\includegraphics[width=0.5\linewidth]{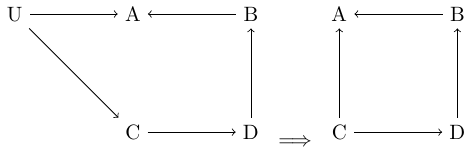}
	\caption{Transformation of a DAG (left) into an ancestral graph (right) by marginalizing over the set of hidden nodes $U$.
		The conditional independence (CI) relationships are preserved; for example, in the left-hand-side (l.h.s.) DAG, $C \perp B \mid D$ holds, and this CI statement is maintained in the right-hand-side (r.h.s.) ancestral graph.
		Similarly, $A \perp D \mid B, C$ is valid in both representations.}
	\label{fig:ex_dag2ancestral}
\end{figure}

Following \citet[Definition~4.2.1]{spirtes2002}, suppose we are given a DAG $D = (V, E)$ and a set $U \subset V$ of nodes that are to be hidden (i.e., unobserved variables).
The corresponding ancestral acyclic directed mixed graph (ancestral ADMG) $D'$ is generated as follows~\Cref{algo:dag2ancestral}.
%For each hidden node $h \in U$, do:
%\begin{enumerate}
%\item For each pair $d_1, d_2 \in \mathrm{ch}(h)$ such that $d_1, d_2 \notin U$ and $d_1 \notin \mathrm{ne}(d_2)$ (to avoid repetitive edges):
%	\begin{enumerate}
%\item If $d_1 \notin \mathrm{an}(d_2)$ but $d_2 \in \mathrm{an}(d_1)$, add the edge $d_1 \rightarrow d_2$. This prevents almost-directed cycles and avoids ambiguity in the ancestral relationship.
%\item If neither $d_1$ is an ancestor of $d_2$ nor $d_2$ is an ancestor of $d_1$, add the edge $d_1 \leftrightarrow d_2$.
%\end{enumerate}
%\item For each $a \in \mathrm{pa}(h)$ and $d \in \mathrm{ch}(h)$ such that $a \notin \mathrm{ne}(d)$, add the edge $a \rightarrow d$.
%\end{enumerate}
%
\begin{algorithm}[H]\caption{DAG to Ancestral Graph}\label{algo:dag2ancestral}
\begin{algorithmic}[1]
  \For{each hidden node \(h \in U\)}
\For{each pair \(d_1, d_2 \in \mathrm{ch}(h)\) such that \(d_1, d_2 \notin U\) and \(d_1 \notin \mathrm{ne}(d_2)\)
  (to avoid repetitive edges)
}
  \If{\(d_1 \notin \mathrm{an}(d_2)\) but \(d_2 \in \mathrm{an}(d_1)\)}
  \State{Add edge \(d_1 \rightarrow d_2\) \Comment{Prevent almost-directed cycles and avoids ambiguity in the ancestral relationship (see also~\Cref{prop:inducing_path_graph_is_not_ancestral}).}}
  \EndIf
  \If{neither \(d_1 \in \mathrm{an}(d_2)\) nor \(d_2 \in \mathrm{an}(d_1)\)}
    \State Add edge \(d_1 \leftrightarrow d_2\)
  \EndIf
\EndFor
\For{each \(a \in \mathrm{pa}(h)\) and \(d \in \mathrm{ch}(h)\) such that \(a \notin \mathrm{ne}(d)\)}
  \State Add edge \(a \rightarrow d\)
\EndFor
\EndFor
\end{algorithmic}
\end{algorithm}

%For each hidden node $h \in U$, do:
%\begin{enumerate}
%	\item For each pair $d_1, d_2 \in \mathrm{ch}(h)$ such that $d_1, d_2 \notin U$ and $d_1 \notin \mathrm{ne}(d_2)$ (to avoid repetitive edges):
%	      \begin{enumerate}
%		      \item If $d_1 \notin \mathrm{an}(d_2)$ but $d_2 \in \mathrm{an}(d_1)$, add the edge $d_1 \rightarrow d_2$.
%		            This prevents almost-directed cycles and avoids ambiguity in the ancestral relationship (see \Cref{prop:inducing_path_graph_is_not_ancestral}).
%		      \item If neither $d_1$ is an ancestor of $d_2$ nor $d_2$ is an ancestor of $d_1$, add the edge $d_1 \leftrightarrow d_2$.
%	      \end{enumerate}
%	\item For each $a \in \mathrm{pa}(h)$ and $d \in \mathrm{ch}(h)$ such that $a \notin \mathrm{ne}(d)$, add the edge $a \rightarrow d$.
%\end{enumerate}
%
\Cref{fig:ex_dag2ancestral} and \Cref{fig:dag2ancestral_u_child_of_o} illustrate two examples of transforming DAGs into ancestral ADMGs.

\begin{figure}
	\centering
	\includegraphics[width=0.5\linewidth]{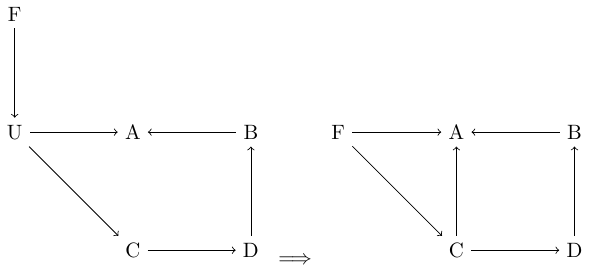}
	\caption{Example of a DAG transformed to an ancestral graph after marginalizing over $U$, where $U$ is the child of an observable $F$.
		Due to the ancestral relationship between nodes $A$ and $C$, no bidirected edge is created.}\label{fig:dag2ancestral_u_child_of_o}
\end{figure}

\begin{remark}
	Ancestral graphs are closed under marginalization with respect to CI statements.
	The above procedure avoids the introduction of almost-directed cycles.
	Once the graph is ancestral, CI assertions can be read off directly.
\end{remark}

%%% Local Variables:
%%% mode: latex
%%% TeX-master: "main"
%%% End:

%\input{experiment} % we don't need this in this paper anymore
\subsection{Illustration}

We illustrate our approach with two concrete examples presented in~\Cref{fig:dag2ancestra_vs_fci_root} and~\Cref{fig:dag2ancestra_vs_fci_intermediate}.
In~\Cref{fig:dag2ancestra_vs_fci_root}, the standard setting is shown where the unobserved confounder is a root variable.

\Cref{fig:dag2ancestra_vs_fci_intermediate} presents an interesting scenario in which the unobserved confounders $X_2$ and $X_5$ are themselves confounded by the observed variable $X_1$.
In this case, the ground truth ancestral ADMG contains no bidirected edges.

\begin{figure}[htbp]
	\centering
	\includegraphics[width=0.95\linewidth]{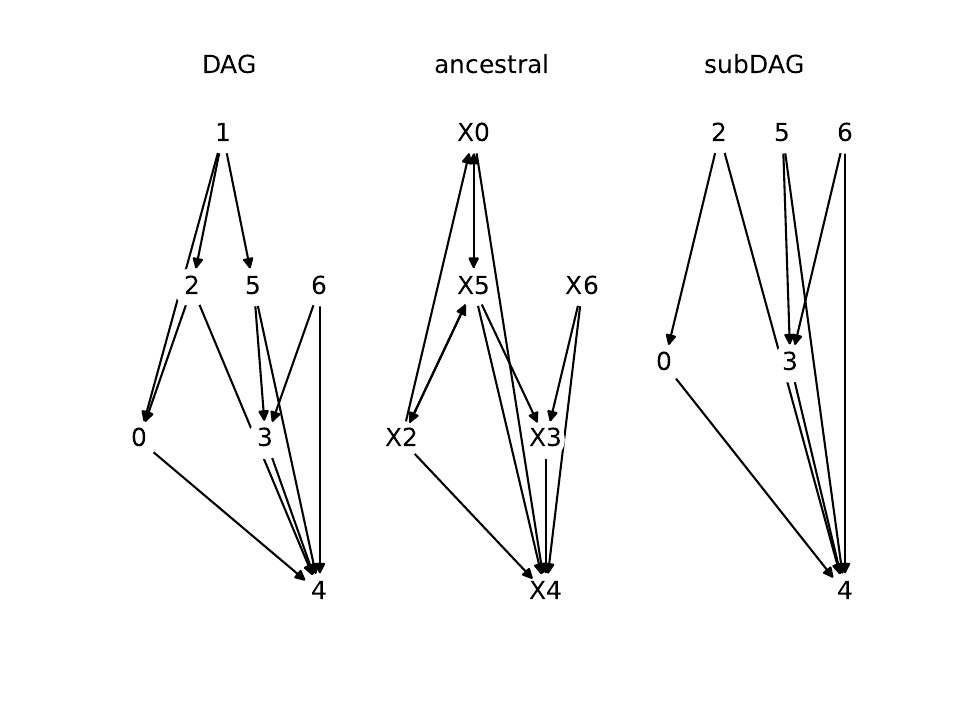}
	\caption{An example DAG generated via our new data generation process and the corresponding ancestral ADMG obtained by hiding the root variable $X_0$.}
	\label{fig:dag2ancestra_vs_fci_root}
\end{figure}

\begin{figure}[htbp]
	\centering
	\includegraphics[width=0.9\linewidth]{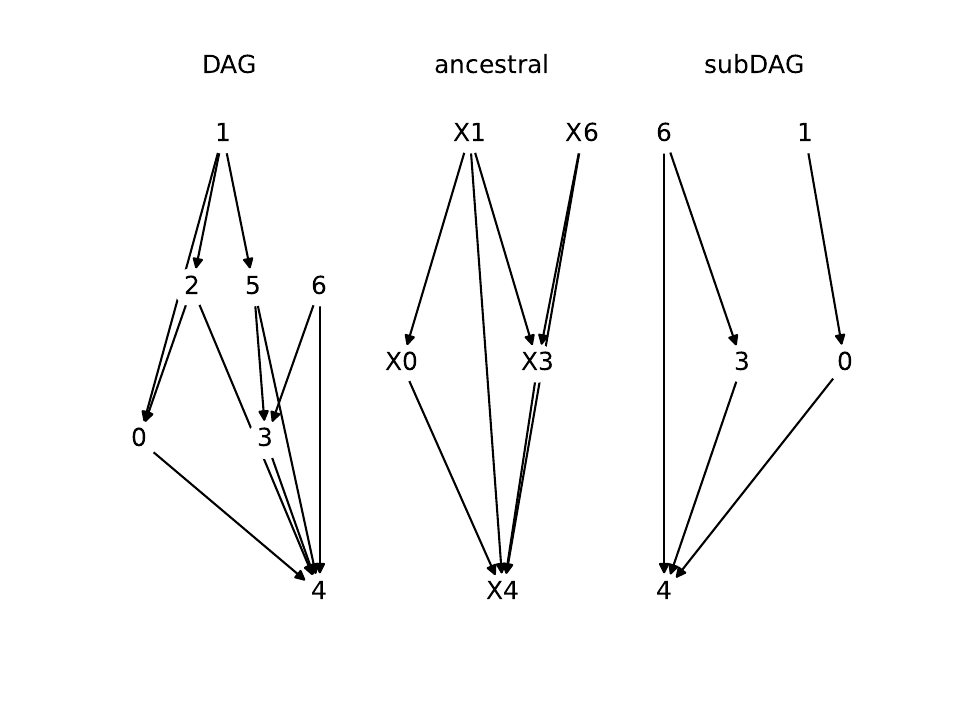}
	\caption{An example DAG generated via our new data generation process and the corresponding ancestral ADMG when hiding the intermediate confounders $X_2$ and $X_5$.}
	\label{fig:dag2ancestra_vs_fci_intermediate}
\end{figure}

\section{Conclusion}\label{sec:conclusion}

We have identified an important limitation in the common practice of imposing a diagonally dominant constraint in the implicit formulation of unobserved confounders, a method widely used in state-of-the-art data synthesis for causal discovery.
This constraint restricts the spectrum of both the correlation and partial correlation matrices.
Together with the restricted bidirected edge sampling method, it prevents full coverage of the positive definite (PD) cone.

To address these issues, we proposed an improved block-hierarchical ancestral DAG generation protocol that explicitly models unobserved confounders by hiding the corresponding columns. For evaluation, we convert the ground truth DAG into ancestral graph to compare the output of causal discovery algorithm.
Our new protocol provides greater control over the graph structure and the unobserved confounders while covering the space of causal models, thereby facilitating more comprehensive evaluation and benchmarking of causal discovery algorithms.

In addition, we analyzed the connection between implicit and explicit unobserved confounding parameterization, illustrates additional parameter constraint the explicit method could introduce, offering a comprehensive understanding of the solution space of causal discovery algorithm with unobserved confounding. 
%%% Local Variables
%%% mode: LaTeX
%%% TeX-master: "main_kdd"
%%% End:

%\newpage

%\section{Statement of Contributions}
% Acknowledgments---Will not appear in anonymized version
%\acks{We thank a bunch of people and funding agency.}

%\subsubsection*{Author Contributions}
%If you'd like to, you may include  a section for author contributions as is done
%in many journals. This is optional and at the discretion of the authors.
%

\begin{acks}
We thank helpful discussions with 
%Dr.~Prof.~
Mathias Drton, Daniela Schokoda, Felix Rios, 
%Dr.~Prof.~
Y. Sammuel Wang, 
%Dr.~Prof.~
Stefan Bauer and Liam Solus.

AM was supported by Novo Nordisk Foundation Grant NNF20OC0062897.
\end{acks}

\section*{Author Contributions}
XS conceived the idea of the paper, developed the theoretical results with proofs independently, and implemented the algorithms. XS and AM implemented the software package and conducted the experiments. AM developed essential points of the paper, including DAG to ancestral graph conversion, as well as co-developed connections between ADMGs and inducing path graphs. AM analyzed and interpreted the findings. PM developed~\Cref{ex:para_un_identify}, assisted XS in development of the block-hierarchical ancestral sampling approach. XS and AM wrote the paper. CM supervised the project, improved the manuscript. %All authors reviewed and approved the final version of the paper.

\bibliography{causalspyne}
\bibliographystyle{acm}

\appendix
\section{General notations}\label{sec:notations}
In this section, general notations used throughout the paper are discussed.
Other notations are left to be introduced in the sections where they are used.

We use $G=(V,E)$ to represent a general graph, and $D=(V,E)$ specifically refers to directed acyclic graph (DAG), where we use $V$ to denote the set of vertex of a graph $G$ (or $D$), and $E$ the set of edges.
$|V|$ (cardinality of set $V$) is used for denoting the number of vertex in the graph.
We use the two terms node and vertex interchangeably in this paper.
A vertex or node in a graph represent random variable (s), which we use capital letter to represent.
We use $U$ to denote the set of unobserved (unmeasured) variables, $O$ to denote the set of observed (measured) random variable, in short observable (s).
$u, v, i, j, k$, as well as integers and $a, b, c, d$, are used to index a vertex (node) in a graph.
Let $Y$ represents a random variable vector, then $Y_u, Y_v, Y_i,Y_j$ represents the $u,v,i,j$ component of the random vector.
We also use $X$ to represent a random vector.
In graphs with directed edges, let $\cdot$ be a vertex or set of vertex, $pa(\cdot)$ is used to represent the set of parents of $\cdot$, $ch(\cdot)$ is to represent the set of children of $\cdot$.
Let $ne(\cdot)=pa(\cdot)\cup ch(\cdot)$ denote the neighbor set.
We use $de(\cdot)$ to represent the set of descendant, $an(\cdot)$ to represent the set of ancestors.
We use $\rho$ to represent a path in a graph, as a sequence of vertex connected one after another, either directed or undirected.
For the directed case, a path start with the source and ends with the sink following the arrow heads.
A vertex (node) in a graph without parent is called a root (source) variable in this paper.

We use $P$ to represent a distribution, and $\mathcal{D}$ to represent a distribution for a specific random vector.
Graph can encode conditional independence (CI) assertions for distributions.
Let $Y_1 \perp Y_2 | S$ denote $Y_1$ and $Y_2$ are conditionally independent when conditioned on variable set $S$.
We use $CI(G)$ ($CI(D)$) to represent the set of CI assertions.
So $CI(P)$ represents the set of CI assertions for a distribution $P$.

We use $W$ to represent the adjacency matrix, to represent a matrix block, we use ${[W]}_J$ where $J$ is an index set for subsetting the rows and column of the matrix $W$, similarly, we use ${[W]}_{J_1, J_2}$ where $J_1$ subset the rows and $J_2$ subset the columns.
To explicitly denote the size of a matrix or matrix block, we use ${[W]}_{|J|}$ for a square matrix of size $|J|$, where $|J|$ is the cardinality of the set $J$.
We use ${[\cdot]}_{|J| \times 1}$ to indicate a column vector of length $J$ where $\cdot$ is the symbol to represent this column vector.
We use $I$ to indicate an identity matrix, with its size to be inferred from the context.

We use $\equiv$ to represent always equal and $:=$ to denote a definition equal.
We use $\epsilon$ to represent idiosyncratic noise variable~\citep{wangdrton23} and $\xi$ for the latent confounder (especially when the unobserved confounder is root variable).
When from the context clear, we use the same two symbols to represent observations (instances) of these random vectors.

%%% Local Variables:
%%% mode: latex
%%% TeX-master: "main"
%%% End:

\section{Preliminaries}\label{sec:preliminaries}
%\subsection{Graph theory}\label{sec:graph_theory}
\subsection{DAGs representation of CI statements}\label{sec:dag_ci}
Terminologies and notations of this section follows~\Cref{sec:notations}.

A graph offers an equivalent class partition for distributions via encoding CI assertions. Let $G$ be a DAG, $P$ be a distribution it represents.
\begin{itemize}
	\item The tuple $(G, P)$ satisfies causal Markov condition if $Y \perp \{V \setminus \left( \mathrm{de}(Y) \cup \mathrm{pa}(Y) \right) \} \mid \mathrm{pa}(Y)$ \\
 In this case, $CI(G) \subset CI(P)$. $G$ is an I-map of $P$. This is also termed the local Markov property.
	\item The tuple $(G,P)$ satisfied minimality condition if any subgraph of $G$ does not satisfy causal Markov condition.
 %(no decomp simplification)
	\item The tuple $(G, P)$ satisfies the faithful condition if $CI(P) \subset CI(G)$
	\item $P$ is faithful to $G$ if $(G, P)$ satisfies Markov and faithful condition.
        \item Two DAGs $G_1$ and $G_2$ are Markov equivalent if the set of distributions faithful to them are the same.
\end{itemize}
Causal Markov condition implies global Markov properties like Markov blanket. Faithfulness implies minimaility, but not the opposite, since there can be missing CI assertions. 
%\begin{figure}
%	\centering
%\includegraphics[width=0.5\linewidth]%{figs/tikzcd_sigma_algebra.pdf}
%\caption{Definition of distribution}\label{fig:enter-label}
%\end{figure}
\subsection{Problems of DAG representation when there exist unobserved confounder}\label{sec:dag_problem}
Causal sufficiency means no unmeasured (unobserved) common causes or the confounder takes only one value, which is usually violated in practice. As an example, in \Cref{fig:spirtes02} (reproduced from~\citep{spirtes2001causation}), $U$ is an unobserved confounder in the DAG of the left-hand side of~\Cref{fig:spirtes02}, which is hidden in the incomplete DAG on the right-hand side of~\Cref{fig:spirtes02}. Then we encounter the following problems:
\begin{itemize}
	\item the incomplete DAG on the right-hand side of~\Cref{fig:spirtes02} imposes extra CI statement:\@ $Y_1 \perp Y_2$
	\item fork $Y1 \not \perp Y2$ in l.h.s.
	      DAG reject $Y1 \perp Y2$ in r.h.s.
	      DAG
	\item $A1 \perp \{A2, Y2\}$ (violation in DAG on the r.h.s.: $A_1 \rightarrow Y_2$)
	      %\item $A2 \perp \{A1, Y1\}$
\end{itemize}
\begin{figure}[h!]
	\centering
	\includegraphics[width=0.9\linewidth]{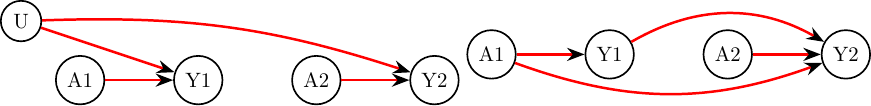}	\caption{problems of DAG representation when there is unobserved variable: variable renamed after~\cite{spirtes2002}}\label{fig:spirtes02}
\end{figure}
It turns out DAGs can not efficiently represent CI statements when there are unobserved confounding which lead to ancestral graphs in the next section.

\subsection{Ancestral graph}\label{sec:ancestral_graph}
\newcommand{\mdag}{G}
\newcommand{\ancestorize}{an}
In this section, we start with a series of essential definitions which lead to ancestral graph.
\begin{definition}[bi-directed-sibling]
	If $A \leftrightarrow B$, we call $A, B$ are bi-directed-siblings, which correspond to the concept of spouse in~\citep{spirtes2002,zhang2008causal}.
\end{definition}

\hypertarget{hyper_target_inducing_path}{} % do not display target
\begin{definition}[Inducing path $\induceP(A,B|O;\mdag)$ relative to variable subset $O$ of DAG $\mdag$~\citep{spirtes2001causation}]\label{def:inducing_path} Let $O$ be the set of observables. Let $L=V(\mdag) \setminus O$ be the set of latent variables (unobserved). 
  Let $\upath{}$ (a set of observables and/or unobserved variables) be an undirected path (a path with undirected edges and maybe directed edges as well) that connects $A\in O$ and $B\in O$: $\upath=\induceP(A,B|O;\mdag)$, If
	\begin{enumerate}
		\item $\forall e \in O\cap \upath \setminus \{A, B\}$, $e$ is a collider
		\item if $v \in \upath$ is a collider (which can be unobserved), $v$ is an ancestor of $A$ or $B$.
                \item A single edge is always defined to be an inducing path relative to any set $O$ or $L$~\citep{zhang2008causal}.
	\end{enumerate}
For an example of inducing path, see~\Cref{fig:ex_inducing_path}.
\end{definition}
\begin{figure}
	\centering
	\includegraphics[width=0.5\linewidth]{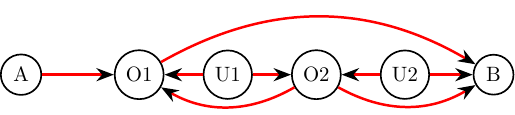}
	\caption{Example of inducing path for a DAG $G$ with $O=\{A,O1,O2,B\}$, $L=\{U1, U2\}$: $\upath{}_p(A,B|O, G)=\{A, O1, U1, O2, U2, B\}$,  in Definition~\ref{def:inducing_path}}\label{fig:ex_inducing_path}
\end{figure}

\begin{remark}
Inducing path $\induceP(A,B|O;\mdag)$ persist (thus we use subscript $p$ of $\upath$) d-connection of $A,B$ regardless of unobserved variable in $G$: $\not\exists S \subset O$, s.t.~$A\perp B | S$. See Thm 6.1 in~\citep{spirtes2001causation}.
\end{remark}

\hypertarget{hyper_target_inducing_path_graph}{} % do not display target
\begin{definition}[Inducing path graph $\ipg(\mdag)$ for DAG $\mdag$~\citep{spirtes2001causation}]
	In $\ipg(\mdag)$, edge $A\to B$ exists $\iff$ there is an inducing path $\induceP(A, B|O;\mdag)$. There are the following edge types
	\begin{itemize}
		\item $A\to B \in G^{'}$: $\exists \induceP(A,B|O;\mdag)$ is out of $A$ and into $B$ in $\mdag$.
		\item $A\leftrightarrow B \in G^{'}$: $\exists \induceP(A,B|O;\mdag)$ is into $A$ and into $B$ in $\mdag$.
		      A double edge in the inducing path graph implies at least two inducing paths between the two variables in question.
	\end{itemize}
\end{definition}

\begin{figure}[h!]
	\centering
	\includegraphics[width=0.5\linewidth]{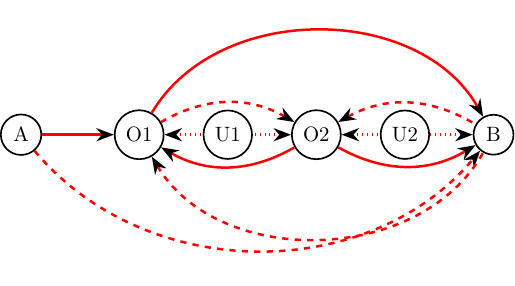}
        \caption{inducing path graph for DAG in~\ref{fig:ex_inducing_path}: Solid and dotted (for edges emitting from unobserved variables) edges are from the original DAG, with dashed edges added indicating inducing paths.}\label{fig:inducing_path_graph}
\end{figure}

\begin{definition}[active path ${\upath}_a$]
	Given set of variables $S$ as condition, the path ${\upath}_a(A, B|S)$ is active between $A, B$, if
	\begin{itemize}
		\item $\forall$ non-collider of ${\upath}_a$ $\not\in S$ (no intercept variable).
		\item $\forall$ collider in ${\upath}_a$ ancestor a variable in $S$.
	\end{itemize}
\end{definition}

\begin{remark}
  An inducing path $\upath_p(A,B|O,G)$  between $A, B$ relative to the observable set $O$ for a graph $D$ is an active path $\upath_a(A, B|S\subset O\setminus\{A, B\})$ for any subset $S\subset O$. From an algorithm point of view, starting from a complete graph, if $A$ and $B$ are d-separated by any subset $S$ of $O$, then there is no inducing path between $A, B$, thus the edge between $A$ and $B$ is removed from the graph (step B of Causal Inference Algorithm in~\citet{spirtes2001causation}).
\end{remark}

\begin{definition}[m-separation]m-separation extends d-separation to the mixed graph: if there is no active path $\upath_a(A,B|S)$, then $S$ m-separates $A, B$
\end{definition}

\begin{definition}[Ancestral~\cite{evans2023latent}]
	An ordinary (single edge) mixed graph without any undirected edge is \emph{ancestral} if its directed part is acyclic, and no vertex
	is an ancestor of any of its bi-directed-siblings; it is \emph{maximal} if every pair of vertices that are not adjacent
	satisfy \(m\)-separation~\citep{spirtes2002}(extension of d-separation to mixed graph, thus m-separation). % \begin{todo}or a nested constraint\end{todo}.
\end{definition}

\begin{figure}[h!]
	\centering
	\includegraphics[scale=0.8]{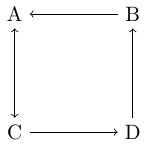}
	\includegraphics[scale=0.8]{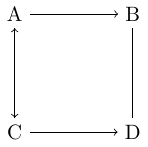}
        \caption{Examples of non-ancestral graphs, renamed after~\citep{spirtes2002}. The left example forms an almost-directed-cycle, see~\Cref{prop:inducing_path_graph_is_not_ancestral}. The right example violates the requirements that undirected end vertex can not have arrow head.}\label{fig:eg_non_ancestral}
\end{figure}

\begin{remark}[examples of non-ancestral]
	For instance, in the left part of~\Cref{fig:eg_non_ancestral}, $A, C$ are bi-directed-siblings connected via bi-directed edge (a.k.a. spouse in the terminology of~\citep{spirtes2002,zhang2008causal}), but there exist a directed path $C\rightarrow~D\rightarrow~B\rightarrow~A$, so this is not ancestral ($C$ is an ancestor of its bi-directed-sibling $A$).
	In the terminology of~\citep{zhang2008causal}, this is called almost-directed-cycle.

	In the right part of~\Cref{fig:eg_non_ancestral}, the graph is non-ancestral since no arrow head is allowed at $B$ since $B$ has undirected connection with $D$.
\end{remark}

\begin{definition}[Maximal ancestral graphs~\citep{zhang2008causal, Ali_Richardson_Spirtes_2009}]
	An ancestral graph is maximal if there is no inducing path between any two non-adjacent vertex.
        A DAG is also a MAG. Alternatively, for any pair of non-adjacent vertices $A, B$, there exists a set $S$, s.t.~$A\not\perp B|S$, i.e.~$S$ connectes $A, B$(If there is inducing path, then $\not\exists$ separation set for $A$ and $B$, Thm 6.1 in~\citet{spirtes2001causation}).
\end{definition}

\begin{proposition}
  A maximal ancestral graph (MAG) \textbf{without undirected edges} is an inducing path graph with observable set $O$ being all variables in the MAG.
\end{proposition}

\begin{proof}
%\todo{XS: is this correct?}
When an ancestral graph $\ancestorize(G)$ is maximal ($\ancestorize(G)$ differs from $G$ in that $G$ is DAG, $\ancestorize(G)$ has only observed vertex of $G$), non-adjacency between $A, B$ indicate there is no induced path between the two variables $A, B$, thus in the induced path graph $G^{'}$, the two nodes $A, B$ are not adjacent.
%For adjacency node pairs of the ancestral graph, we can define $A\rightarrow B$ without intermediate variable to be an inducing path.
\end{proof}

\begin{proposition}\label{prop:inducing_path_graph_is_not_ancestral}
An \hyperlink{hyper_target_inducing_path_graph}{inducing path graph} is not necesarily a maximal ancestral graph.
\end{proposition}
\begin{proof}
  An inducing path graph only contains bidirected and directed edges with arrowhead indicating an inducing path, which satisfies the maximal arguments. So it suffices to check if almost-directed-cycle exsits in inducing path graph, or alternatively speaking, bidirected end vertex should not ancestor each other. Consider~\Cref{ex:inducing_path_graph_almost_d_cycle} as a counter example.
\end{proof}

\begin{example}\label{ex:inducing_path_graph_almost_d_cycle}
Consider the DAG below with unobserved confounder $U$:
\begin{itemize}
  \item $A\leftarrow U\rightarrow B$
  \item $A\rightarrow C \rightarrow B$
\end{itemize}
The corresponding inducing path graph below has almost directed cycle.
\begin{itemize}
  \item $A\leftrightarrow B$
  \item $A\rightarrow C \rightarrow B$
\end{itemize}
\end{example}
\begin{remark}
Almost-directed-cycle only causes ambiguity in ancestral relationship, but does not violate the acyclic property of DAG.
For example, in inducing path graph $G^{'}$, if $A,B$ are bi-directed-siblings, i.e. $A\leftrightarrow B \in G^{'}$, corresponding to in original DAG: $B\leftarrow U_2\rightarrow O_1 \leftarrow U_1 \rightarrow A$ and $O_1\leftrightarrow A$ in $G^{'}$.

if $\exists$ almost-directed-cycle from $A$ to $B$, then $\exists C$, s.t.
$A\rightarrow C\rightarrow B$ in $G^{'}$, corresponding to in original DAG: $A\rightarrow U_3 \rightarrow C\rightarrow U_4 \rightarrow B$ which does not violate the acyclic property of DAG.
\end{remark}

\begin{remark}
A partially oriented inducing path graph is not a graph.
\end{remark}

\begin{remark}[Ancestral ADMG]
Ancestral ADMGs:~In this paper, we only deal with~\emph{mixed graphs} containing directed and bidirected~(but no undirected) edges. The intersection of ADMGs and ancestral graph is called ancestral ADMG.\end{remark}

\begin{definition}[O-Markov equivalent between DAGs]\label{def:o-markov_equiv}
	Two DAGS $D_1(O, U)$ and $D_2(O,U)$ are O-Markov equivalent if they represent the same set of CI statements with respect to observables.
\end{definition}

\begin{definition}[Distribution equivalent]
Two DAGS $D_1(V)$ and $D_2(V)$ are distribution equivalent if they code the same local Markov properties.
\end{definition}

\begin{definition}[O-Distribution equivalent between DAGs]\label{def:o-markov_equiv}
Two DAGS $D_1(O, U)$ and $D_2(O,U)$ are O-Distribution equivalent if they code the same local Markov properties with respect to observables.
\end{definition}

\section{Connections between implicit and explicit formulation of confounding}\label{sec:connect_implicit_explicit}
\subsection{Existence of equivalent adjacency matrices}\label{sec:o-markv-equiv}
The explicit reformulation of the implicit unobserved confounding model in~\Cref{eq:root_confounder_data_gen} is a special block structure of zero block for $\xi$ in the adjacency matrix $W^{'}$. It is then a natural question if there exist another adjancency matrix $W^{''}$ other than $W^{'}$ in~\Cref{eq:root_confounder_data_gen} such that the marginal distribution of $O$ remains invariant.

We use covariance matrix to characterize the distribution, so the first question we are interested
is if there is congruent transformation between covariance matrices.

\begin{restatable}{proposition}{LabelRestatablePropAlgEqInvarianceSigmaTransform}\label{coro:alg_eq_invariance_sigma_transform}
%\begin{proposition}\label{coro:alg_eq_invariance_sigma_transform}
Let $\Sigma^{'}$ be an arbitrary covariance matrix with the following block structure:
  \begin{equation}
    \Sigma^{'}=\begin{bmatrix} 
      &\Sigma_{11} &\Sigma_{12}\\
      &\Sigma_{21} &\Sigma_{22}
    \end{bmatrix}\label{eq:sigma_block}
  \end{equation}

Let the congruent transformation matrix be 
\begin{equation}
Q=\begin{bmatrix}&Q_{11} &Q_{12}\\&Q_{21} &Q_{22}
\end{bmatrix}
\end{equation}

Any congruent transformation matrix $Q$ to $\Sigma^{'}$ that keeps the top-left block invariant, is a solution to the matrix equation %including $Q$ in~\Cref{eq:congruent_mat_schur_complement}, 
\begin{align}
&{\Sigma}_{11}\nonumber\\=&  Q_{11}^T\Sigma_{11}Q_{11}+Q_{21}^T\Sigma_{21}Q_{11} + Q_{11}^T\Sigma_{12}Q_{21}+Q_{21}^T\Sigma_{22}Q_{21}
\end{align}
%\end{proposition}
\end{restatable}

\begin{proof}
  Let $\Sigma^{''}$ be the resulting matrix from the congruent transform
Then
\begin{align}\label{eq:}
  \Sigma^{''}:=Q^T\Sigma^{'} Q &=\begin{bmatrix}&Q_{11}^T &Q_{21}^T\\&Q_{12}^T &Q^T_{22}\end{bmatrix}\begin{bmatrix}
    \Sigma_{11} & \Sigma_{12} \\
    \Sigma_{21} & \Sigma_{22}
  \end{bmatrix}  \begin{bmatrix}&Q_{11} &Q_{12}\\&Q_{21} &Q_{22}
\end{bmatrix}
\end{align}
The top-left block of the resulting matrix is
\begin{equation}
  {[Q^T\Sigma^{'} Q]}_{11}= Q_{11}^T\Sigma_{11}Q_{11}+Q_{21}^T\Sigma_{21}Q_{11} + Q_{11}^T\Sigma_{12}Q_{21}+Q_{21}^T\Sigma_{22}Q_{21}\label{eq:algebraic_eq_sigma_invariance}
\end{equation}
Letting the top-left block of the resulting matrix be $\Sigma_{11}$, we have

\begin{equation}
{\Sigma}_{11}=  Q_{11}^T\Sigma_{11}Q_{11}+Q_{21}^T\Sigma_{21}Q_{11} + Q_{11}^T\Sigma_{12}Q_{21}+Q_{21}^T\Sigma_{22}Q_{21}\label{eq:alg_eq_invariance_sigma_transform}
\end{equation}
\end{proof}

\begin{corollary}\label{coro:alg_eq_invariance_sigma_transform_special_solution_q21_zero}
  The upper triangular block matrix with $Q_{21}=0, Q_{11}=I$ is a special solution to~\Cref{eq:alg_eq_invariance_sigma_transform}. 
\end{corollary}
\begin{proof}
This solution satifies~\Cref{eq:alg_eq_invariance_sigma_transform} regardless of the value of $Q_{12}$.
\end{proof}

\subsubsection{Schur complement solution}
\begin{lemma}\label{lemma:schur_complement_congruent_transformation}
Let $\Sigma^{'}$ be an arbitrary covariance matrix with the following block structure:
  \begin{equation}
    \Sigma^{'}=\begin{bmatrix} 
      &\Sigma_{11} &\Sigma_{12}\\
      &\Sigma_{21} &\Sigma_{22}
    \end{bmatrix}\label{eq:sigma_block}
  \end{equation}
  There exists a congruent transformation with matrix $Q^{-1}$ to $\Sigma^{'}$ that does not change the top-left block of the covariance matrix.
\begin{align}
  Q&=\begin{bmatrix}&Q{11} &Q{12}\\&Q{21} &Q{22}\end{bmatrix}\\
   &=\begin{bmatrix}&I &\Sigma_{11}^{-1}\Sigma_{12}\\&0 &I\end{bmatrix}\label{eq:congruent_mat_schur_complement}
\end{align}

\end{lemma}
\begin{proof}
  Since both $\Sigma^{'}$ and $\Sigma_{11}$ is p.d., they are also invertible.
Denote the Schur complement of $\Sigma_{11}$ in $\Sigma$ as 

\begin{equation}
[\Sigma / \Sigma_{11}]= \Sigma_{22}-\Sigma_{21}\Sigma_{11}^{-1}\Sigma_{12}
\end{equation}
In addition, we have
\begin{align}
  Q^T&=\begin{bmatrix}&Q{11}^T &Q{11}^T\\&Q{12}^T &Q{22}^T\end{bmatrix}\\
     &=\begin{bmatrix}&I &0\\&{[\Sigma_{11}^{-1}\Sigma_{12}]}^T &I\end{bmatrix}\\
     &=\begin{bmatrix}&I &0 \\&\Sigma_{21}\Sigma_{11}^{-1} &I\end{bmatrix}
\end{align}

With Schur complement, we can decompose $\Sigma$:
  \begin{align}
    &\begin{bmatrix} 
      &\Sigma_{11} &\Sigma_{12}\\
      &\Sigma_{21} &\Sigma_{22}
      \end{bmatrix} \nonumber\\
    =& \begin{bmatrix} 
      &I &0\\
      &\Sigma_{21}\Sigma_{11}^{-1} &I\end{bmatrix} \begin{bmatrix}
      &\Sigma_{11} &0\\
      &0 &[\Sigma / \Sigma_{11}] \end{bmatrix} 
      \begin{bmatrix}&I &\Sigma_{11}^{-1}\Sigma_{12}\\&0 &I\end{bmatrix}\\=&Q^T\Sigma^{''} Q
  \end{align}
  where
\begin{equation}
  \Sigma^{''}=
\begin{bmatrix}
      &\Sigma_{11} &0\\
      &0 &[\Sigma / \Sigma_{11}] \end{bmatrix} =
      Q^{-T}\Sigma^{'} Q^{-1}\label{eq:sigma_schur_com_transform}
\end{equation}
$\Sigma^{''}$ keeps the $\Sigma_{11}$ block invariant.
\end{proof}
\begin{remark}
$Q$ is not orthogonal.
\begin{align}\label{eq:}
  Q^{T}Q&=\begin{bmatrix}&I &0 \\&\Sigma_{21}\Sigma_{11}^{-1} &I\end{bmatrix}\begin{bmatrix}&I &\Sigma_{11}^{-1}\Sigma_{12}\\&0 &I\end{bmatrix}\\
      &=\begin{bmatrix}&I &\Sigma_{11}^{-1}\Sigma_{12}\\&\Sigma_{21}\Sigma_{11}^{-1} &\Sigma_{21}\Sigma_{11}^{-2}\Sigma_{12}+I\end{bmatrix}
\end{align}
Due to the block-upper-triangular structure of $Q$, $Q$ is invertible.
\begin{equation}
  Q^{-1}= \begin{bmatrix}&I &-\Sigma_{11}^{-1}\Sigma_{12}\\&0 &I\end{bmatrix}\neq Q^T
\end{equation}

Since $Q$ is invertible, we have
\begin{align}
{(Q^{T})}^{-1}={(Q^{-1})}^T=\begin{bmatrix}&I &0\\&-\Sigma_{21}\Sigma_{11}^{-T} &I\end{bmatrix}
\end{align}

\end{remark}

\begin{corollary}
$Q^T$ ($Q$ defined in~\Cref{eq:congruent_mat_schur_complement}) instead of $Q^{-1}$ as congruent transformation matrix also preserves one diagonal block of the covariance matrix, thus is also a solution to~\Cref{coro:alg_eq_invariance_sigma_transform}.
\end{corollary}
\begin{proof}
This is a special case of~\Cref{coro:alg_eq_invariance_sigma_transform_special_solution_q21_zero}.
\end{proof}

\begin{remark}
  Take inverse of~\Cref{eq:sigma_schur_com_transform}:
\begin{align}
  {(\Sigma^{''})}^{-1}&=
{\begin{bmatrix}
      &\Sigma_{11} &0\\
      &0 &[\Sigma / \Sigma_{11}] \end{bmatrix}}^{-1} =
      Q{(\Sigma^{'})}^{-1} Q^{T}\label{eq:sigma_schur_com_transform_inv}\\
      &=\begin{bmatrix}
      &{(\Sigma_{11})}^{-1} &0\\
      &0 &{([\Sigma / \Sigma_{11}])}^{-1} \end{bmatrix}
\end{align}
while
\begin{align}
{(\Sigma^{'})}^{-1}={\begin{bmatrix} 
      &\Sigma_{11} &\Sigma_{12}\\
      &\Sigma_{21} &\Sigma_{22}\end{bmatrix}}^{-1}=\begin{bmatrix} 
      &{[{(\Sigma^{'})}^{-1}]}_{11} &{[{(\Sigma^{'})}^{-1}]}_{12}\\
      &{[{(\Sigma^{'})}^{-1}]}_{21} &{([\Sigma / \Sigma_{11}])}^{-1}\end{bmatrix}
\end{align}
So the congruent transform with $Q^T$ of ${(\Sigma^{'})}^{-1}$ as a precision matrix, the lower-right block ${([\Sigma / \Sigma_{11}])}^{-1}$
is preserved.
\end{remark}

\begin{corollary}
Let $X_1, \ldots, X_p$ be multivariate Gaussian whose covariance matrix is $\Sigma^{'}$.
There exists a linear transformation to multivariate Gaussian that does not change the top-left block of the covariance matrix.
\end{corollary}
\begin{proof}
  Let the transpose or inverse transpose of $Q$ in~\Cref{eq:congruent_mat_schur_complement} operate on $X_1, \ldots, X_p$, whose covariance matrix is $\Sigma^{'}$, then 
  $Q^{-T}[X_1, X_2, \ldots, X_p]$ is a linear transformation of $X_1, X_2, \ldots, X_p$ that does not change the top-left block of the covariance matrix.
  Since the covariance of $Q^{-T}[X_1, X_2, \ldots, X_p]$ is $Q^{-T}\Sigma Q^{-1}$, the top-left block of the covariance matrix is $\Sigma_{11}$ according to~\Cref{eq:sigma_schur_com_transform,lemma:schur_complement_congruent_transformation}.
\end{proof}

\LabelRestatableThmTransformW*
\begin{proof}\label{proof:thm_w_transform}
Let $W_1^{''}$ be the adjacency matrix of a DAG, congruent and similar to adjacency matrix $W^{'}$ via invertible matrix $Q^T$ (instead of $Q$), where
\begin{equation}
Q=\begin{bmatrix}&I &\Sigma_{11}^{-1}\Sigma_{12}\\&0 &I\end{bmatrix}(\text{from}~\Cref{eq:congruent_mat_schur_complement})
\nonumber
%\label{eq:q}
\end{equation}
is the congruent transformation matrix in~\Cref{eq:congruent_mat_schur_complement}. 

Write~\Cref{eq:root_confounder_data_gen} as
  \begin{align}
	Y^{'}:=
	\begin{pmatrix}
		 & Y_O   \\
		 & \xi
	\end{pmatrix}
	                                        & =
	\begin{bmatrix}
		 & W_o & \Lambda \\
		 & 0   & 0
	\end{bmatrix}
	\begin{pmatrix}
		 & Y_O   \\
		 & \xi
	\end{pmatrix}
	+
	\begin{pmatrix}
		 & \epsilon_O \\
		 & \xi
	\end{pmatrix}
	= W^{'}
	Y^{'}+\epsilon^{'}\nonumber\\
&~\Cref{eq:root_confounder_data_gen}\nonumber\\
&=Q^{-T}W_1^{''}Q^{T}Y^{'}+\epsilon^{'}\label{eq:sem_adjacency_congruent}
\end{align}
where \begin{equation}
 W^{'}=\begin{bmatrix}
     & W_o & \Lambda \\
     & 0   & 0 \end{bmatrix}=Q^{-T}W_1^{''}Q^{T}\label{eq:adjacency_matrix_congruent_def}
\end{equation}
and
\begin{equation}
W^{''}_1=Q^{T}W^{'}Q^{-T}
\end{equation}
Due to the block-upper-triangular structure of $Q$, $Q$ is invertible.
\begin{equation}
  Q^{-1}= \begin{bmatrix}&I &-\Sigma_{11}^{-1}\Sigma_{12}\\&0 &I\end{bmatrix}\neq Q^T
\end{equation}
Since $Q$ is invertible, we have
\begin{align}
{(Q^{T})}^{-1}={(Q^{-1})}^T=\begin{bmatrix}&I &-\Sigma_{11}^{-1}\Sigma_{12}\\&0 &I\end{bmatrix}
\end{align}
equivalently to~\Cref{eq:sem_adjacency_congruent}, we have
\begin{align}
  Q^{T}Y^{'}&=W_1^{''}Q^{T}Y^{'}+Q^{T}\epsilon^{'}
\end{align}
Let 
\begin{equation}
Y^{''}=Q^{T}Y^{'}, \epsilon^{''}=Q^{T}\epsilon^{'}
\end{equation}
We have
\begin{align}
Y^{''}&=W^{''}_1Y^{''}+\epsilon^{''} \label{eq:sem_new_adj}
\end{align}
i.e.
\begin{align}
  Y^{''}&=Q^{T}Y^{'}=\begin{bmatrix}&I &0\\ &\Sigma_{21}\Sigma_{11}^{-1} &I \end{bmatrix}\begin{pmatrix}
     & Y_O   \\
     & \xi
  \end{pmatrix}\\
      &=\begin{pmatrix}
     & Y_O\\
     & \Sigma_{11}^{-1}\Sigma_{21}Y_O +\xi
  \end{pmatrix}
\end{align}
We also have
\begin{align}
\Sigma^{''}_{12}=&\mathbb{E}\left(Y_O (\Sigma_{11}^{-1}\Sigma_{21}Y_O +\xi)^T\right)\\
=&\mathbb{E}\left(Y_OY_O^T\Sigma_{21}^T\Sigma_{11}^{-T} + Y_O\xi^T\right)\\
=&\Sigma_{11}\Sigma_{12}\Sigma_{11}^{-T}+\Sigma_{12}
\end{align}

and
\begin{align}
  \epsilon^{''}&=Q^{T}\epsilon^{'}=\begin{bmatrix}&I &0\\ &\Sigma_{21}\Sigma_{11}^{-1} &I \end{bmatrix}\begin{pmatrix}
     & \epsilon_O   \\
     & \xi
\end{pmatrix}\label{eq:epsilon_prime_prime_q_transform}\\
      &=\begin{pmatrix}
     & \epsilon_O\\
     & \Sigma_{11}^{-1}\Sigma_{21}\epsilon_O +\xi
  \end{pmatrix}
\end{align}
From~\Cref{eq:adjacency_matrix_congruent_def}, we have
\begin{align}
  W_1^{''}&=Q^{T}W^{'}Q^{-T}=Q^T\begin{bmatrix}
     & W_o & \Lambda \\
     & 0   & 0 \end{bmatrix}Q^{-T}\\
     &=\begin{bmatrix}&I &0\\ &\Sigma_{21}\Sigma_{11}^{-1} &I \end{bmatrix}\begin{bmatrix}
     & W_o & \Lambda \\
     & 0   & 0 \end{bmatrix}
\begin{bmatrix}&I &-\Sigma_{11}^{-1}\Sigma_{12}\\&0 &I\end{bmatrix}\\
&=\begin{bmatrix}
W_o & \Lambda - W_o \Sigma_{11}^{-1} \Sigma_{12} \\
\Sigma_{21}\Sigma_{11}^{-1} W_o & \Sigma_{21}\Sigma_{11}^{-1} (\Lambda - W_o \Sigma_{11}^{-1} \Sigma_{12})
\end{bmatrix}
\end{align}

\paragraph{Adjacency matrix $W_2^{''}$:}
If we change the transformation matrix from $Q^T$ to $Q^{-T}$, we have
\begin{align}
  Y^{''}&=Q^{-T}Y^{'}=
\begin{bmatrix}&I &0\\&-\Sigma_{21}\Sigma_{11}^{-T} &I\end{bmatrix}
  \begin{pmatrix}
     & Y_O   \\
     & \xi
  \end{pmatrix}\\
      &=\begin{pmatrix}
     & Y_O\\
     & \xi -\Sigma_{21}\Sigma_{11}^{-T}Y_O
  \end{pmatrix}
\end{align}
and
\begin{align}
  \epsilon^{''}&=Q^{-T}\epsilon^{'}=\begin{bmatrix}&I &0\\&-\Sigma_{21}\Sigma_{11}^{-T} &I\end{bmatrix}
\begin{pmatrix}
     & \epsilon_O   \\
     & \xi
\end{pmatrix}\label{eq:epsilon_prime_prime_q_transform}\\
      &=\begin{pmatrix}
     & \epsilon_O\\
     & \xi-\Sigma_{11}^{-1}\Sigma_{21}\epsilon_O 
  \end{pmatrix}
\end{align}
the off-diagonal block of $\Sigma^{''}$ is
\begin{align}
\Sigma^{''}_{12}=&\mathbb{E}\left(Y_O{(\xi -\Sigma_{21}\Sigma_{11}^{-T}Y_O)}^{T}\right)\\
=&\mathbb{E}\left(Y_O\xi^T -Y_OY_O^T\Sigma_{11}^{-1}\Sigma_{12}\right)\\
=&\Sigma_{12}-\Sigma_{11}\Sigma_{11}^{-1}\Sigma_{12}=0
\end{align}
The corresponding adjacency matrix $W_2^{''}$ is
\begin{align}
  W^{''}_2&=Q^{-T}W^{'}Q^{T}=Q^{-T}\begin{bmatrix}
     & W_o & \Lambda \\
     & 0   & 0 \end{bmatrix}Q^{T}\\
     &=
\begin{bmatrix}&I &-\Sigma_{11}^{-1}\Sigma_{12}\\&0 &I\end{bmatrix}
\begin{bmatrix}
     & W_o & \Lambda \\
     & 0   & 0 \end{bmatrix}
    \begin{bmatrix}&I &0\\ &\Sigma_{21}\Sigma_{11}^{-1} &I \end{bmatrix}
\\
&=\begin{bmatrix}
W_o + \Lambda \Sigma_{21} \Sigma_{11}^{-1} & \Lambda \\
0 & 0
\end{bmatrix}\end{align}
\end{proof}

\begin{remark}\label{remark:q_idio_transform}
The $Q_{12}$ component in~\Cref{eq:congruent_mat_schur_complement} will mix $\xi$ into the second partition of $\epsilon^{''}$ in~\Cref{eq:epsilon_prime_prime_q_transform}.
Unless $Q$ (equivalently $Q^T$) is orthogonal, i.e. $[I, {\Sigma_{11}}^{-1}\Sigma_{12}]$ orthogonal, or $\Sigma_{12}=0$, the resulting equivalent idiosyncratic variable can not be independent. 
%\todo{the transformation $Q$ does not make the diagonal structure of $\Omega$.}
But we have
\begin{align}
  \Sigma_{12}&=\mathbb{E}(Y_O, \xi^T)=\mathbb{E}(I-W)^{-1}\xi_Y \xi^T
  =(I-W)^{-1}\Omega\neq0
\end{align}
\end{remark}

\Cref{remark:q_idio_transform} lead us to~\Cref{coro:similar_transform_requirements}.

\begin{corollary}\label{coro:similar_transform_requirements}
Among all adjacency matrices $Q^TW^{'}Q^{-T}$ similar to $W^{'}$ in~\Cref{eq:root_confounder_data_gen}, only $Q^T$ being diagonal matrix, or $Q \sqrt{diag\left(\mathbb{E}(\epsilon^{'}{\epsilon^{'}}^T)\right)}$ being orthogonal can pertain the diagonal structure of idosyncratic noise covariance $\Omega^{''}$.
\end{corollary}
\begin{proof}
Rewrite~\Cref{eq:epsilon_prime_prime_q_transform}
\begin{equation}
\epsilon^{''}=Q^T\epsilon^{'}=Q^T\begin{pmatrix}
& \epsilon_O \\
& \xi
\end{pmatrix}\label{eqdef:epsilon_double_prime}
\end{equation}
\begin{equation}
\mathbb{E}(\epsilon^{''}{\epsilon^{''}}^T)=Q^T\mathbb{E}(\epsilon^{'}{\epsilon^{'}}^T)Q 
\end{equation}
Since $\mathbb{E}(\epsilon^{'}{\epsilon^{'}}^T)$ is diagonal,
\begin{align}
  {[Q^T\mathbb{E}(\epsilon^{'}{\epsilon^{'}}^T)Q]}_{i,j}&=\sum_{k_1, k_2} Q^T_{i,k_1}{[\mathbb{E}(\epsilon^{'}{\epsilon^{'}}^T)]}_{k_1, k_2}{Q}_{k_2, j}\\
                                                        &=\sum_{k=k_1=k_2} Q^T_{i,k}{[\mathbb{E}(\epsilon^{'}{\epsilon^{'}}^T)]}_{k}{Q}_{k, j}\\
                                                        &=\sum_{k=k_1=k_2} Q_{k,i}{[\mathbb{E}(\epsilon^{'}{\epsilon^{'}}^T)]}_{k}{Q}_{k,j}
\end{align}
which is weight inner product between column $i$ and $j$ of $Q$.
To satisfy $\forall i\neq j$, ${[Q^T\mathbb{E}(\epsilon^{'}{\epsilon^{'}}^T)Q]}_{i,j}
=\sum_{k} Q_{k,i}{[\mathbb{E}(\epsilon^{'}{\epsilon^{'}}^T)]}_{k}{Q}_{k,j}=0
$, one solution is $Q$ being diagonal.
Another solution is $Q \sqrt{diag\left(\mathbb{E}(\epsilon^{'}{\epsilon^{'}}^T)\right)}$ being orthogonal, where the $i$th column of $Q$ get multiplied by $\sqrt{\mathbb{E}(\epsilon_i^2)}$.
\end{proof}
\subsubsection{Other solution}
Are there other transform matrix $Q$ that can keep the $\Sigma_{11}$ structure without $Q_{21}=0$ in~\Cref{coro:alg_eq_invariance_sigma_transform_special_solution_q21_zero}?
\begin{lemma}\label{lemma:alg_eq_invariance_sigma_transform_simplified}
Let $\Sigma_{22}=I$, then~\Cref{eq:alg_eq_invariance_sigma_transform} becomes
\begin{equation}
{\Sigma}_{11}=Q_{11}^T\Sigma_{11}Q_{11}+Q_{21}^T\Sigma_{21}Q_{11} + Q_{11}^T\Sigma_{12}Q_{21}+Q_{21}^TQ_{21}
\end{equation}
Let $Q_{11}=I$
\begin{equation}
0=Q_{21}^T\Sigma_{21} + \Sigma_{12}Q_{21}+Q_{21}^TQ_{21}\label{eq:alg_eq_invariance_sigma_transform_simplified}
\end{equation}
A solution w.r.t. $Q_{21}$ is
\begin{equation}
Q_{21}=(S-I)\Sigma_{21}
\end{equation}
where $S$ is an arbitrary orthonormal matrix.
\end{lemma}

\begin{proof}
According to~\citet{jean_marie2025}, via plus and minus trick,~\Cref{eq:alg_eq_invariance_sigma_transform_simplified} can be written as
\begin{equation}
{(Q_{21}+\Sigma_{21})}^T (Q_{21}+\Sigma_{21})=\Sigma_{21}^T\Sigma_{21}
\end{equation}
Set $(Q_{21}+\Sigma_{21})=S\Sigma_{21}$ where $S$ is an arbitrary orthonormal matrix, or
$Q_{21}=(S-I)\Sigma_{21}$, we have 
\begin{equation}
\Sigma_{21}^TS^TS\Sigma_{21}=\Sigma_{21}^T\Sigma_{21} 
\end{equation}
\end{proof}

\begin{corollary}\label{coro:alg_eq_invariance_sigma_transform_simplified_non_triangular}
Let the congruent transformation matrix be 
\begin{align}
  Q&=\begin{bmatrix}&Q{11} &Q{12}\\&Q{21} &Q{22}\end{bmatrix}\\
   &=\begin{bmatrix}&I &Q_{12} \\&(S-I)\Sigma_{21} &Q_{22}\end{bmatrix}\label{eq:solution_linear_transform_adjacency}
\end{align}
where $S$ is an arbitrary orthonormal matrix, and $Q_{12}, Q_{22}$ can be arbitrary value.
%Q_{12}!={[(S-I)\Sigma_{21}]}^T?
Then the congruent transformation of $\Sigma$ with operator $Q$ leaves the top-left block of the covariance matrix $\Sigma_{11}$ unchanged when $\Sigma_{22}=I$.
\end{corollary}
\begin{proof}
This is a consequence of~\Cref{coro:alg_eq_invariance_sigma_transform} and~\Cref{lemma:alg_eq_invariance_sigma_transform_simplified}.
\end{proof}
\subsubsection{Future work}
We leave it for future work to varify if there exist a solution in~\Cref{coro:alg_eq_invariance_sigma_transform_simplified_non_triangular} that fits the requirements of~\Cref{coro:similar_transform_requirements}. Since \Cref{coro:alg_eq_invariance_sigma_transform_simplified_non_triangular} is a special case when $\Sigma_{22}=I$, we can also explore the general case when $\Sigma_{22}\neq I$.
Thus, we leave the following conjecture as future work.
%\todo{XS:how about set $Q_{11}\neq I$??}
%\todo{XS:prove the conjecture via special case of only 1 unobserved confounder.}

%\todo{Can we prove the conjecture at the end of this section which can make the diagonal structure of $\Omega$?or can we prove the negation of this conjecture using~\Cref{coro:similar_transform_requirements}?}

\begin{conjecture}
There exists a general adjacency matrix
	\begin{align}
          W^{''} & :=
		\begin{bmatrix}
			 & {[W]}_{J_O} & {[W]}_{J_{O}, J_{U}} \\  
                         &{[W]}_{J_{U}, J_{O}}\not\equiv 0 &{[W]}_{J_U} \not\equiv 0
		\end{bmatrix}\\
                 &= QW^{'}Q^{-1}
	\end{align}
where $W^{'}$ is the adjacency matrix in~\Cref{eq:root_confounder_data_gen}.
$W^{''}$ with independent idiosyncratic variable, represents the same observed data in terms of invariant $\Sigma_{11}=\Sigma_Y=\mathbb{E}(YY^{T})$.
i.e. \textbf{the corresponding $\Omega^{''}=\mathbb{E}(\epsilon^{''}{\epsilon^{''}}^T)$ of $W^{''}$ is diagonal}, where $\epsilon^{''}=Q\epsilon^{'}$ is defined in~\Cref{eqdef:epsilon_double_prime},
and $Q$ is a solution to the matrix equation
\begin{align}
&{\Sigma}_{11}\nonumber\\=&  Q_{11}^T\Sigma_{11}Q_{11}+Q_{21}^T\Sigma_{21}Q_{11} + Q_{11}^T\Sigma_{12}Q_{21}+Q_{21}^T\Sigma_{22}Q_{21}
\end{align} from~\Cref{coro:alg_eq_invariance_sigma_transform} and satisfy~\Cref{coro:similar_transform_requirements}.
\end{conjecture}

\subsection{Extra advantage of the explicit method: introducing extra constraint and parametric unidentifiability}\label{sec:trek_constraint_star_canonical}
In the linear Gaussian case, the class of O-distributions (distributions w.r.t.~observables) from the two data synthesis method are O-Markov indistinguishable, they can still differ in parameterization. 

%To illustrate this, we start with the following definitions and give examples to show extra constrains the explicit formulation can introduce.
\begin{definition}[canonical DAG]
	Given an ancestral graph $G$, where every bidirected edge $k\leftrightarrow l$ is replaced by a new vertex $i$ along with the directed edges $i\rightarrow k$ and $i\rightarrow l$.
	We denote this new DAG by $D(G)$, the canonical DAG with respect to an ancestral graph~\citep{spirtes2002}.
\end{definition}
Different DAGs can end up producing the same ancestral graph $G$ upon marginalization.
Let one of such a DAG be $D_1(G)$.
Although $D_1(G)$ and $D(G)$ (the canonical DAG as defined above) satisfy the same set of polynomial constraints among the covariances of the observed variables, they could still differ as models.
For instance, there could be certain polynomial inequalities (semi-algebraic constraints) that are satisfied by one of the DAGs and not by the other, illustrated in~\Cref{ex:para_un_identify}.
\begin{example}\label{ex:para_un_identify}
	Let $D_1$ be the DAG on 4 nodes with the edge set $\{1\rightarrow 2, 1\rightarrow 3, 1\rightarrow 4\}$.
	Upon marginalization of the vertex $1$, we obtain the ancestral graph $G(D_1;1)$ (we use this notation to indicate the ancestral graph via hiding variable $1$ from $D_1$) with vertices $2,3,4$ and bidirected edges $\{2\leftrightarrow 3, 2\leftrightarrow 4, 3\leftrightarrow 4\}$.
	Now, in order to obtain $D(G(D_1;1))$ (the canonical DAG), we add 3 new confounders $1,5$ and $6$, so that we get the DAG on 6 vertices with the edge set $\{1\rightarrow 2, 1\rightarrow 3, 5\rightarrow 3, 5\rightarrow 4, 6\rightarrow 2, 6\rightarrow 4\}$.
	Both DAGs do not satisfy any polynomial constraints in the covariances of the observed variables $2,3$ and $4$.

	Let $\omega_i:=\Omega_{ii}$, using the trek rule on $D_1$, we can explicitly state the covariances in terms of the error variances $\omega$ and structural coefficients $\Lambda$ (instead of $W$ in~\Cref{def:trek_monomial} to conform to~\Cref{eq:oeqbolambdaxi}) in the following way:
	\begin{eqnarray*}
		\Sigma_{23}&=&\omega_1\Lambda_{12}\Lambda_{13} \\
		\Sigma_{24}&=&\omega_1\Lambda_{12}\Lambda_{14} \\
		\Sigma_{34}&=&\omega_1\Lambda_{13}\Lambda_{14}.
	\end{eqnarray*}
	Similarly, for $D(G(D_1;1))$, we get:
	\begin{eqnarray*}
		\Sigma_{23}&=&\omega_1\Lambda_{12}\Lambda_{13} \\
		\Sigma_{24}&=&\omega_6\Lambda_{62}\Lambda_{64} \\
		\Sigma_{34}&=&\omega_5\Lambda_{53}\Lambda_{54}.
	\end{eqnarray*}
	As all the error variances are non-negative, we get the following inequality constraint for $D_1$: \[ \Sigma_{23}\Sigma_{24}\Sigma_{34}= \omega_1^3\Lambda_{12}^2\Lambda_{13}^2\Lambda_{14}^2 \geq 0, \] which is not necessarily true for $D(G(D_1;1))$ as \[ \Sigma_{23}\Sigma_{24}\Sigma_{34}= \omega_1 \omega_5 \omega_6 \Lambda_{12}\Lambda_{13}\Lambda_{62}\Lambda_{64}\Lambda_{53}\Lambda_{54}, \] could be negative for certain choice of the structural coefficients.
\end{example}

The above example shows that the corresponding models might differ even though they satisfy the same constraints on the covariances of the observed variables.
We do however, conjecture that even though the two models are different, the model obtained from $D_1$ is always contained in the one obtained from $D(G(D_1; 1))$.
This is because, one could potentially assign specific values to certain variance and edge parameters of $D(G(D_1; 1))$ to obtain the exact same parametrization as $D_1$.
For instance, in the above example, use superscript to distinguish the two DAGs, we could set
\begin{eqnarray*}
	\omega_1^{D(G(D_1; 1))}=\omega_5^{D(G(D_1; 1))}=\omega_6^{D(G(D_1; 1))}=\omega_1^{D_1} \text{ and} \\
	\lambda_{62}^{D(G(D_1; 1))}=\lambda_{12}^{D_1}, \lambda_{53}^{D(G(D_1; 1))}=\lambda_{13}^{D_1}, \\
	\lambda_{54}^{D(G(D_1; 1))}=\lambda_{64}^{D(G(D_1; 1))}=\lambda_{14}^{D_1}
\end{eqnarray*}
to obtain the same parametrization of $D_1$.

This implies even if the two DAGs discussed here are O-Markov equivalent, i.e., the set of CIs w.r.t.\ only observables are the same, their parametrization can be different, with one satisfying a semi-algebraic constraint but not the other.
%But still, it is not possible to distinguish the two different parameterization from observed data?   Does this have a connection to arid ADMG and bow-free ADMG at all?  (ancestral ADMG $\subset$ arid ADMG
%$\subset$ Bow free ADMG.

%\end{document}

%%% Local Variables:
%%% mode: latex
%%% TeX-master: "main"
%%% End:

%\input{spirtes02}
\section{Other mathematical proofs}
We restate lemmas, corollaries, propositions, and theorems in the main text in this section, using the same numbering,
%~\Cref{lemma:par_corr_precision_mat},~\Cref{lemma:spectrum_inv_sigma_congruent}~
and append their proofs afterwards:

\LabelRestatableExplicitDGP*
\begin{proof}\label{proof:explicit_dgp}
	Define
	\[
		\epsilon_Y' = \Lambda\,\xi + \epsilon_O.
	\]
	Then
	\begin{align}
		\Omega' & = \mathbb{E}(\epsilon_Y'\,{\epsilon_Y'}^T) \nonumber                                                          \\[1mm]
		        & = \mathbb{E}\left( (\Lambda\,\xi + \epsilon_O)(\Lambda\,\xi + \epsilon_O)^T \right) \nonumber                 \\[1mm]
		        & = \Lambda\,\mathbb{E}(\xi\xi^T)\,\Lambda^T + \mathbb{E}(\epsilon_O\,\epsilon_O^T).\label{eq:corr_idio_decomp}
	\end{align}
	By \Cref{prop:explicit_always_psd}, the matrix $\Omega'$ is positive definite.
	Moreover, condition~\eqref{eq:constraint_mat_Lambda} ensures that $\Omega'$ has nonzero off-diagonal entries.
	Thus, the explicit representation~\eqref{eq:oeqbolambdaxi} is equivalent to the implicit formulation given by \Cref{eq:wplp,eq:corr_y} with a non-diagonal, positive definite $\Omega$.
        Note that the nonzero off-diagonal correspond to $\sum_u\Lambda_{v_1, u}\Lambda_{v_2, u}$.
	%\todo{Explain further how the nonzero off-diagonals are enforced, if needed.}
\end{proof}

\LabelRestatableCoroExplicitDAG*
\begin{proof}\label{proof:coro_explicit_dag}
Note that the block of zeros in the lower left of $W'$ indicates that the confounders $\xi$ do not have parents in the observable subsystem and, thus, appear as root nodes.
	Moreover, the presence of $\Lambda$ as the upper-right block shows the direct influence of $\xi$ on $Y_O$, which, upon inducing the covariance $\Omega'$, accounts for the nonzero off-diagonals.
	This reformulation emphasizes that the DAG is, in fact, bipartite with $\xi$ as one partition and $Y_O$ (along with any additional noise variables) as the other.

	Under our convention, the entry $W_{i,j}$ corresponds to an edge directed from node $j$ into node $i$.
	Thus, having a zero block for the rows corresponding to $\xi$ indicates that $\xi$ has no incoming edges, i.e., it is a source node.
	Consequently, the full DAG that includes $\xi$ is bipartite, with $\xi$ in one partition and $Y_O$ (and $\epsilon_O$) in the other.
\end{proof}

\LabelRestableOnePlusPD*
\begin{proof}\label{proof:1plusexplicit_always_psd}
Diagonally dominant matrix is positive definite.
\end{proof}
\LabelRestatableExplicitPD*
\begin{proof}\label{proof_prop:explicit_always_psd}
	Since the components of $\xi$ are mutually independent, $\mathbb{E}(\xi\xi^T)$ is a positive diagonal matrix with the $i$th diagonal entry equal to $\operatorname{Var}(\xi_i)$.
	Hence,
	\[
		0 \prec \mathbb{E}(\xi\xi^T).
	\]
	By Sylvester's law of inertia~\citep[Theorem 4.3.7]{horn2012matrix}, any congruent transformation preserves definiteness.
	Therefore,
	\[
		\Lambda\,\mathbb{E}(\xi\xi^T)\,\Lambda^T \succ 0.
	\]
	Similarly, since the components of $\epsilon_O$ are mutually independent, $\mathbb{E}(\epsilon_O\,\epsilon_O^T)$ is a positive diagonal matrix with the $i$th diagonal entry given by $\operatorname{Var}([\epsilon_O]_i)$, so
	\[
		0 \prec \mathbb{E}(\epsilon_O\,\epsilon_O^T).
	\]
	Since the sum of two positive definite matrices remains positive definite, it follows that for every nonzero vector $x\in \mathbb{R}^{|J_O|}$,
	\[
		x^T \Lambda\,\mathbb{E}(\xi\xi^T)\,\Lambda^T\,x + x^T\mathbb{E}(\epsilon_O\,\epsilon_O^T)x > 0.
	\]
	Hence,
	\[
		\Omega = \Lambda\,\mathbb{E}(\xi\xi^T)\,\Lambda^T + \mathbb{E}(\epsilon_O\,\epsilon_O^T) \succ 0.
	\]
\end{proof}

%%% Local Variables:
%%% mode: LaTeX
%%% TeX-master: "main_kdd"
%%% End:

\LabelRestatableTrekRule*
\begin{proof}\label{proof:trek_rule}
Since
\begin{align}
{(I-W)}^{-1} &= I + W + W^2 + \ldots \\
{(I-W)}^{-T} &= I + W^T + {(W^T)}^2 + \ldots
\end{align}
we have
\begin{equation}
  {(I-W)}^{-1}\Omega{(I-W)}^{-T} = \Omega + \sum_{r_1=0,r_2=0}^{\infty} W^{r_1}\Omega {(W^T)}^{r_2}
\end{equation}
The covariance between $u,v$ is given by 
\begin{align}
  \mathbb{E}[X_uX_v] &= \Omega_{u,v} + \sum_{r_1=0,r_2=0}^{\infty}{[W^{r_1}\Omega {(W^T)}^{r_2}]}_{u,v}\\
                     &=\Omega_{u,v} + \sum_{r_1=0,r_2=0}^{\infty}\sum_{t_1, t_2}{{[W^{r_1}]}_{u,t_1}\Omega_{t_1,t_2}{[{(W^T)}^{r_2}]}_{t_2,v}}\\
                     &=\Omega_{u,v} + \sum_{r_1=0,r_2=0}^{|V|}\sum_{t_1, t_2}{{[W^{r_1}]}_{u,t_1}\Omega_{t_1,t_2}{[{(W^T)}^{r_2}]}_{t_2,v}}
\end{align}
where the summation reduces from $\infty$ to maximum $v$ due to nilpotency of $W$, where $W^{|V|}=0$ indicates that no connection from a vertex to itself is allowed which ensures acyclicity of the graph. $W^{|V|}=0$ further implies that $W^{|V|}v=\lambda^{|V|}v=0$ so all eigenvalues of $W$ are $0$.

From variable $t_1$, a path travels $r_1$ steps to reach vertex $u$ and starting from $t_2$ a path lead to $v$ in $r_2$ steps. This is a trek from $u$ to $v$ with top $t_1, t_2$ being confounded if $t_1\neq t_2$.

The summation over all possible treks between $u,v$ is given by $\sum_{r_1=0,r_2=0}^{\infty}\sum_{t_1, t_2}$ can be written as 

\[\sum_{trek(u,v;k=(t_1,t_2)|D) \in T(u,v|D)} m_{trek(u,v;k=(t_1,t_2)|D)}\]

with
\begin{align}
  m_{trek(u,v;k=(t_1,t_2)|D)}&=\Omega_{t_1,t_2}\prod_{s_1, s_2\in~trek(u,v;k|D)} W_{s_1,s_2}\\
                   &={{[W^{r_1}]}_{u,t_1}\Omega_{t_1,t_2}{[{(W^T)}^{r_2}]}_{t_2,v}}
\end{align}
\end{proof}

\LabelRestatableLemmaAlterResidualParCorr*
\begin{proof}\label{proof:par_corr_exclude_one_var}
Let $Y_{\setminus i}$ denote the design matrix $Y$ with the $i^\text{th}$ column removed, and similarly let $Y_{\setminus \{i,j\}}$ denote the matrix $Y$ with both the $i^\text{th}$ and $j^\text{th}$ columns removed.
	Then, the residuals for the regression of $Y_i$ and $Y_j$ on $V \setminus \{i,j\}$ are given by
	\begin{align}
		\hat{\epsilon}_{i|i,j} & = Y_i - Y_{\setminus \{i,j\}}\Bigl({Y_{\setminus \{i,j\}}^T Y_{\setminus \{i,j\}}}\Bigr)^{-1}Y_{\setminus \{i,j\}}^T Y_i \\
                                       &=Y_i - Y_{\setminus \{i,j\}}\Bigl(\hat{\Sigma}_{\setminus\{i,j\}}\Bigr)^{-1}{\hat{\Sigma}}_{\setminus \{i,j\}; i} \\
                                       &=(I-P_{\setminus i,j})Y_i\\
		\hat{\epsilon}_{j|i,j} & = Y_j - Y_{\setminus \{i,j\}}\Bigl({Y_{\setminus \{i,j\}}^T Y_{\setminus \{i,j\}}}\Bigr)^{-1}Y_{\setminus \{i,j\}}^T Y_j\\
                                       &=Y_j - Y_{\setminus \{i,j\}}\Bigl(\hat{\Sigma}_{\setminus\{i,j\}}\Bigr)^{-1}{\hat{\Sigma}}_{\setminus \{i,j\}; j}\\
                                       &=(I-P_{\setminus i,j})Y_j\\
	\end{align}
        where
\begin{equation}
P_{\setminus i,j}=Y_{\setminus \{i,j\}}\Bigl({Y_{\setminus \{i,j\}}^T Y_{\setminus \{i,j\}}}\Bigr)^{-1}Y_{\setminus \{i,j\}}^T
\end{equation} is the symmetric projection operator onto the columns space of $Y_i$ (premultiply).

Projection operator has the property
\begin{equation}
(I-P_{\setminus i,j})(I-P_{\setminus i,j})=I-P_{\setminus i,j}
\end{equation}
due to $P_{\setminus i,j}P_{\setminus i,j}=P_{\setminus i,j}$.

        We have
        \begin{align}
&\hat{\epsilon}_{i|i,j}^T \hat{\epsilon}_{j|i,j}\nonumber\\
&=Y_i^T(I-P_{\setminus i,j})(I-P_{\setminus i,j})Y_j\\
&=Y_i^T(I-P_{\setminus i,j})Y_j
%=&Y_iY_j-\nonumber\\
% &Y_{\setminus \{i,j\}}\Bigl(\hat{\Sigma}_{\setminus\{i,j\}}\Bigr)^{-1}{\hat{\Sigma}}_{\setminus \{i,j\}; i}Y_j-\nonumber\\
% &Y_iY_{\setminus \{i,j\}}\Bigl(\hat{\Sigma}_{\setminus\{i,j\}}\Bigr)^{-1}{\hat{\Sigma}}_{\setminus \{i,j\}; j}+\nonumber\\
%          &Y_{\setminus \{i,j\}}\Bigl(\hat{\Sigma}_{\setminus\{i,j\}}\Bigr)^{-1}{\hat{\Sigma}}_{\setminus \{i,j\}; i}Y_{\setminus \{i,j\}}\Bigl(\hat{\Sigma}_{\setminus\{i,j\}}\Bigr)^{-1}{\hat{\Sigma}}_{\setminus \{i,j\}; j}
%\nonumber
        \end{align}

	Similarly, the residual for the regression of $Y_i$ on $V \setminus \{i\}$ is
\begin{align}
		% X\beta=Y, \beta=(X^TX)^{-1}X^TY, \hat{Y}=X\beta=X(X^TX)^{-1}X^TY
		  & \hat{\epsilon}_{i|i}\nonumber                                                                                                                      \\
		= & Y_i-\hat{Y}_i^{V\setminus\{i\}}                                                                                                                    \\
		= & Y_i-{Y}_{\setminus i}{({Y}_{\setminus i}^T{Y}_{\setminus i})}^{-1}{Y}_{\setminus i}^TY_i                                                           \\
		= & Y_i-[{Y}_{\setminus \{i,j\}}, Y_j]{({[{Y}_{\setminus \{i,j\}}, Y_i]}^T{[{Y}_{\setminus \{i,j\}}, Y_i]})}^{-1}{[{Y}_{\setminus \{i,j\}}, Y_i]}^TY_i \\
		= & Y_i-[{Y}_{\setminus \{i,j\}}, Y_j]{\begin{bmatrix}
			                                        & {Y}_{\setminus \{i,j\}}^T{Y}_{\setminus \{i,j\}} & {Y}_{\setminus \{i,j\}}^TY_i \\
			                                        & Y_i^T{Y}_{\setminus \{i,j\}}                     & Y_i^TY_i
		                                       \end{bmatrix}}^{-1}{[{Y}_{\setminus \{i,j\}}, Y_i]}^TY_i \label{eq:project_only_1var}\\
		= & Y_i-[{Y}_{\setminus \{i,j\}}, Y_j]{\begin{bmatrix}
                                                                &{\hat{\Sigma}}_{\setminus \{i,j\}} & {\hat{\Sigma}}_{\setminus \{i,j\}; i} \\
                                                                & {\hat{\Sigma}}_{\setminus \{i,j\}; i}^T                     & \hat{\sigma}_i
                  \end{bmatrix}}^{-1}\begin{bmatrix}
                                                                &{\hat{\Sigma}}_{\setminus \{i,j\}; i}\\ &\hat{\sigma}_i
                  \end{bmatrix}\\
                  =&(I-P_{\setminus i})Y_i
	\end{align}
        where
        \begin{align}
          P_{\setminus i}&=Y_{\setminus i}\Bigl({Y_{\setminus i}^T Y_{\setminus i}}\Bigr)^{-1}Y_{\setminus i}^T\\          
                         &=
[{Y}_{\setminus \{i,j\}}, Y_j]{\begin{bmatrix}
			                                        & {Y}_{\setminus \{i,j\}}^T{Y}_{\setminus \{i,j\}} & {Y}_{\setminus \{i,j\}}^TY_i \\
			                                        & Y_i^T{Y}_{\setminus \{i,j\}}                     & Y_i^TY_i
                                                      \end{bmatrix}}^{-1}{[{Y}_{\setminus \{i,j\}}, Y_i]}^T\\
    &= 
[{Y}_{\setminus \{i,j\}}, Y_j]\nonumber\\
    &{\Bigl(L\begin{bmatrix}
    &Y_{\setminus \{i,j\}}^TY_{\setminus \{i,j\}} &0\\
      &0 &[(Y_{\setminus i}^TY_{\setminus i}) \Bigg/(Y_{\setminus \{i,j\}}^TY_{\setminus \{i,j\}})] \end{bmatrix} 
      L^T\Bigr)}^{-1}\nonumber\\
                     &{[{Y}_{\setminus \{i,j\}}, Y_i]}^T\\
                     &=[{Y}_{\setminus \{i,j\}}, Y_j]L^{-T}\nonumber\\
    &{\begin{bmatrix}
    &Y_{\setminus \{i,j\}}^TY_{\setminus \{i,j\}} &0\\
      &0 &([Y_{\setminus i}^TY_{\setminus i}) \Bigg/ (Y_{\setminus \{i,j\}}^TY_{\setminus \{i,j\}} )] \end{bmatrix} }^{-1}\nonumber\\
    &L^{-1}{[{Y}_{\setminus \{i,j\}}, Y_j]}^T\label{eq:operator_projection_setminus_i}
        \end{align}
        % todo{NOTE: XS, I made a mistake, should be $y_i$}
with the Schur complement of $Y_{\setminus \{i,j\}}^TY_{\setminus \{i,j\}}$ in 
$Y_{\setminus i}^TY_{\setminus i}$ as 

\begin{align}
&[(Y_{\setminus i}^TY_{\setminus i}) \Bigg/ (Y_{\setminus \{i,j\}}^TY_{\setminus \{i,j\}})]\nonumber\\
=&Y_i^TY_i-(Y_i^TY_{\setminus\{i,j\}}){(Y_{\setminus \{i,j\}}^TY_{\setminus \{i,j\}})}^{-1}(Y^T_{\setminus\{i,j\}}Y_i)\\
=&Y_i^T(I-P_{\setminus i,j})Y_i
\end{align}

and
  \begin{align}
    L=\begin{bmatrix} 
      &I &0\\
      &Y_i^TY_{\setminus\{i,j\}}{\bigl(Y_{\setminus \{i,j\}}^TY_{\setminus \{i,j\}}\bigr)}^{-1} &I\end{bmatrix} 
\end{align}
Rewrite $L$ as
$$L=\begin{bmatrix} 
      &I &0\\
      &L_{21} &I\end{bmatrix} 
$$
then 
\begin{align}
  L^{-1}&=\begin{bmatrix} 
      &I &0\\
      &-L_{21} &I\end{bmatrix} \\
        &=\begin{bmatrix} 
      &I &0\\
      &-Y_i^TY_{\setminus\{i,j\}}{\bigl(Y_{\setminus \{i,j\}}^TY_{\setminus \{i,j\}}\bigr)}^{-1} &I\end{bmatrix} 
\end{align}
so
\begin{align}
&L^{-1}{[{Y}_{\setminus \{i,j\}}, Y_i]}^T\\
=&\begin{bmatrix} 
      &I &0\\
      &-Y_i^TY_{\setminus\{i,j\}}{\bigl(Y_{\setminus \{i,j\}}^TY_{\setminus \{i,j\}}\bigr)}^{-1} &I\end{bmatrix}
\begin{bmatrix}&Y_{\setminus\{i,j\}}^T\\
&Y_i^T\end{bmatrix}\\
=&\begin{bmatrix}&Y_{\setminus\{i,j\}}^T\\
&-Y_i^TY_{\setminus\{i,j\}}{\bigl(Y_{\setminus \{i,j\}}^TY_{\setminus \{i,j\}}\bigr)}^{-1}Y_{\setminus\{i,j\}}^T+Y_i^T
\end{bmatrix}\\
  =&\begin{bmatrix}&Y_{\setminus\{i,j\}}^T\\
&Y_i^T(I-P_{\setminus i,j})
\end{bmatrix}
\end{align}
Define
\begin{align}
&P_{(I-P_{\setminus i,j})Y_i}\nonumber\\
=&\bigl((I-P_{\setminus i,j})Y_i\bigr){\Bigl(Y_i^T(I-P_{\setminus i,j})Y_i\Bigr)}^{-1}Y_i^T(I-P_{\setminus i,j})\\
=&\bigl((I-P_{\setminus i,j})Y_i\bigr){\Bigl(Y_i^T(I-P_{\setminus i,j})(I-P_{\setminus i,j})Y_i\Bigr)}^{-1}Y_i^T(I-P_{\setminus i,j})
\end{align}

So~\Cref{eq:operator_projection_setminus_i} becomes
%\todo{XS: check if $y_i$}
\begin{equation}
P_{\setminus i}=P_{\setminus i,j} + P_{(I-P_{\setminus i,j})Y_j}~\text{~\citep{antti1983connection}}
\end{equation}
then
\begin{align}
\hat{\epsilon}_{j|j}&=(I-P_{\setminus j})Y_j\\
                    &=(I-P_{\setminus i,j} - P_{(I-P_{\setminus i,j})Y_i})Y_j 
\end{align}
\begin{align}
\hat{\epsilon}_{i|i}&=(I-P_{\setminus i})Y_i\\
                    &=(I-P_{\setminus i,j} - P_{(I-P_{\setminus i,j})Y_j})Y_i
\end{align}

\begin{align}
&\hat{\epsilon}^T_{i|i}\hat{\epsilon}_{j|j}\\
=&Y_i^T(I-P_{\setminus i,j} - P_{(I-P_{\setminus i,j})Y_j})(I-P_{\setminus i,j} - P_{(I-P_{\setminus i,j})Y_i})Y_j\\
=&Y_i^T(I-P_{\setminus i,j})Y_j + Y_i^TP_{(I-P_{\setminus i,j})Y_j}P_{(I-P_{\setminus i,j})Y_i}Y_j - \\
 &Y_i^T(I-P_{\setminus i,j})P_{(I-P_{\setminus i,j})Y_i}Y_j-Y_i^TP_{(I-P_{\setminus i,j})Y_j}(I-P_{\setminus i,j})Y_j
\end{align}

Now the additive factors of $Y_i^T(I-P_{\setminus i,j} - P_{(I-P_{\setminus i,j})Y_j})(I-P_{\setminus i,j} - P_{(I-P_{\setminus i,j})Y_i})Y_j$:

\begin{align}
&Y_i^TP_{(I-P_{\setminus i,j})Y_j}P_{(I-P_{\setminus i,j})Y_i}Y_j\\
=&Y_i^T\bigl((I-P_{\setminus i,j})Y_j\bigr){\Bigl(Y_j^T(I-P_{\setminus i,j})(I-P_{\setminus i,j})Y_j\Bigr)}^{-1}Y_j^T(I-P_{\setminus i,j})\nonumber\\
 &\bigl((I-P_{\setminus i,j})Y_i\bigr){\Bigl(Y_i^T(I-P_{\setminus i,j})(I-P_{\setminus i,j})Y_i\Bigr)}^{-1}Y_i^T(I-P_{\setminus i,j})Y_j \nonumber\\
=&(\hat{\epsilon}_{i|i,j}^T\hat{\epsilon}_{j|i,j})\hat{\epsilon}_{j|i,j}^T\hat{\epsilon}_{i|i,j}(\hat{\epsilon}_{i|i,j}^T\hat{\epsilon}_{j|i,j})/(\hat{\epsilon}_{i|i,j}^T\hat{\epsilon}_{i|i,j} \hat{\epsilon}_{j|i,j}^T\hat{\epsilon}_{j|i,j} )
\end{align}

\begin{align}
&Y_i^T(I-P_{\setminus i,j})P_{(I-P_{\setminus i,j})Y_i}Y_j\\
=&Y_i^T(I-P_{\setminus i,j})\bigl((I-P_{\setminus i,j})Y_i\bigr){\Bigl(Y_i^T(I-P_{\setminus i,j})(I-P_{\setminus i,j})Y_i\Bigr)}^{-1}Y_i^T(I-P_{\setminus i,j})\nonumber\\
 &Y_j\nonumber\\
=&Y_i^T(I-P_{\setminus i,j})Y_j~\text{(last 4 factors, rest cancels to identity)}\\
=&\hat{\epsilon}_{i|i,j}^T\hat{\epsilon}_{j|i,j} 
\end{align}

\begin{align}
&Y_i^TP_{(I-P_{\setminus i,j})Y_j}(I-P_{\setminus i,j})Y_j\\
=&Y_i^T\bigl((I-P_{\setminus i,j})Y_j\bigr){\Bigl(Y_j^T(I-P_{\setminus i,j})(I-P_{\setminus i,j})Y_j\Bigr)}^{-1}Y_j^T(I-P_{\setminus i,j})\nonumber\\
 &(I-P_{\setminus i,j})Y_j\nonumber\\
=&Y_i^T(I-P_{\setminus i,j})Y_j~\text{(first 4 factors, rest cancels to identity)}\\
=&\hat{\epsilon}_{i|i,j}^T\hat{\epsilon}_{j|i,j} 
\end{align}

so
\begin{align}
&\hat{\epsilon}^T_{i|i}\hat{\epsilon}_{j|j}\\
=&\hat{\epsilon}_{i|i,j}^T\hat{\epsilon}_{j|i,j} +(\hat{\epsilon}_{i|i,j}^T\hat{\epsilon}_{j|i,j})\hat{\epsilon}_{j|i,j}^T\hat{\epsilon}_{i|i,j}(\hat{\epsilon}_{i|i,j}^T\hat{\epsilon}_{j|i,j})/(\hat{\epsilon}_{i|i,j}^T\hat{\epsilon}_{i|i,j} \hat{\epsilon}_{j|i,j}^T\hat{\epsilon}_{j|i,j} )
 - \\
 &\hat{\epsilon}_{i|i,j}^T\hat{\epsilon}_{j|i,j} - \hat{\epsilon}_{i|i,j}^T\hat{\epsilon}_{j|i,j} 
\end{align}
This lead to the conclusion stated in~\Cref{eq:def_partial_corr_alt}, see also~\cite{antti1983connection}.
\end{proof}

\LabelRestatableOmegaInvSpectrum*
\begin{proof}\label{proof:inv_omega_diag_dom_spectral_radius}
	According to~\citet{varah1975lower}, if $\Delta_i(\Omega) > 0$ as in~\Cref{def:diag_dom}, then
	\begin{equation}
		\|\Omega^{-1}\|_{\infty} \le \max_{1\le i\le |V|}\frac{1}{\Delta_i(\Omega)},
	\end{equation}
	where the matrix infinity norm for the precision matrix $\Omega^{-1}$ of size $|V|$ is defined as
	\begin{equation}
		\|\Omega^{-1}\|_{\infty} = \max_{x \neq 0} \frac{\|\Omega^{-1}x\|_{\infty}}{\|x\|_{\infty}} = \max_{1\le i\le |V|}\sum_{j=1}^{|V|}\bigl|[\Omega^{-1}]_{i,j}\bigr|.
	\end{equation}

	The Gershgorin disks for $\Omega^{-1}$ are given by
	\[
		D\Bigl([\Omega^{-1}]_{ii}, \sum_{j\neq i} \bigl|[\Omega^{-1}]_{ij}\bigr|\Bigr),
	\]
	and since
	\[
		\sum_{j\neq i} \bigl|[\Omega^{-1}]_{ij}\bigr| \le \max_{1\le i\le |V|}\frac{1}{\Delta_i(\Omega)} - \bigl|[\Omega^{-1}]_{ii}\bigr|,
	\]
	we conclude that the spectral radius of $\Omega^{-1}$ is bounded by $\max_i \frac{1}{\Delta_i(\Omega)}$.
	In the special case of~\Cref{coro:omega_psd} where $\Delta_i(\Omega) > 1$, we have that the right-hand bound is less than $1$, implying
	\[
		\lambda(\Omega^{-1}) < 1.
	\]
%	\todo{XS: I forgot what this paper is about~\citet{johnson2024diagonal}}
\end{proof}

\LabelLemmaPartialCorrPrecision*%restatement of lemma
\begin{proof}\label{proof:par_corr_precision}
\newcommand{\myscalemat}{S}
Let the design matrix be $X=[X_1, \cdots, X_p]$, with observed covariates $\frac{1}{n(n-1)}X^TX$ following Wishart distribution with scale matrix $\myscalemat$.

Treat $X_i$ as a basis vector for the vector space of dimension $n$ where $n$ is the number of observations, following~\citet{partial_corr_precision_dual}, define inner product
\begin{equation}
<X_i,X_j>:=\myscalemat_{i,j} 
\end{equation}
Define
\begin{equation}
  \hat{\beta}_{i,\cdot}=\arg\min_{\beta_{i,j}}||X_i-\sum_{j\neq i}\beta_{i,j}X_{j}||
\end{equation}

Corresponding to % X\beta=Y, \beta=(X^TX)^{-1}X^TY
\begin{equation}
\hat{\beta}_{i,\cdot}=(X_{-i}^TX_{-i})^{-1}X_{-i}^TX_i
\end{equation}
We have
\begin{equation}
||{X}_i-\sum_{j\neq i}\beta_{i,j}{X}_{j}||=\myscalemat_{i,i}+{(\sum_{j\neq i}\beta_{i,j}X_j)}^2-2\sum_{j\neq i}\beta_{i,j}\myscalemat_{i,j}
\end{equation}
with minimum at 
\begin{align}
  <2(\sum_{j\neq i}\hat{\beta}_{i,j}X_j),X_j>-2S_{i,j}&=0\\
  \sum_{k\neq i}\hat{\beta}_{i,k}S_{k,j}-S_{i,j}&=0
%\hat{\beta}_{i,j}=\frac{S_{i,j}}{S_{j,j}} % wrong!
%<X_{-i}, ({X}_i-\sum_{j\neq i}\beta^{*}_{i,j}{X}_{j})>&=0\\
%\sum_{j^{'}\neq i}\myscalemat_{j^{'},i}-\sum_{k\neq i}\sum_{j\neq i} \beta^{*}_{i,j}\myscalemat_{k,j}                                                &=0
\end{align}

Define $X_{i,\circ}$ and $X_i^{*}$ via
\begin{equation}
  X_i-\sum_{j\neq i}\hat{\beta}_{i,j}X_{j}:=X_{i,\circ}:=\lambda_iX_i^{*}, \forall i\label{eq:residual}
\end{equation}
we have
\begin{align}
  <X_i^{*}, X_j>&:=\frac{1}{\lambda_i}<X_i-\sum_{k\neq i}\hat{\beta}_{i,k}X_{k}, X_j
>\\
                &=\frac{1}{\lambda_i}(\myscalemat_{i,j}-\sum_{k\neq i}\hat{\beta}_{i,k}\myscalemat_{k,j})\\
                &=\frac{1}{\lambda_i}(\myscalemat_{i,j}-S_{i,j})\\
                &=0~\forall j\neq i
                \end{align}
via defining
\begin{equation}
<X_i^{*}, X_i>=1
\end{equation}
We have
\begin{equation}
<X_i^{*}, X_j>=\delta_{i,j} \label{eq:bracket_delta}
\end{equation}

According to~\Cref{lemma:par_corr_alter_residual}, define the partial correlation between variable indexed $i$ and $j$ among other variables as
\begin{align}\label{eq:partial_corr}
\rho_{i,j|\cdot}&:=-\frac{<X_{i,\circ},X_{j,\circ}>}{\sqrt{<X_{i,\circ},X_{i,\circ}><X_{j,\circ},X_{j,\circ}>}} \\
                &=-\frac{<X_{i}^{*},X_{j}^{*}>}{\sqrt{<X_{i}^{*},X_{i}^{*}><X_{j}^{*},X_{j}^{*}>}} \end{align}

                Rewrite~\Cref{eq:residual} as
\begin{align}
X_i^{*}&= \sum_{j=1}^{|V|}\beta^{'}_{i,j}X_j\\
\beta^{'}_{i,j}&=-\frac{\hat{\beta}_{i,j}}{\lambda_i}~\forall j\neq i\\
\beta^{'}_{i,i}&=\frac{1}{\lambda_i}
\end{align}

Rewrite~\Cref{eq:bracket_delta} as
\begin{align}
\delta_{i,k}=<X_i^{*},X_k>&=\sum_{j=1}^{|V|}\beta^{'}_{i,j}<X_j,X_k>\\
                          &=\sum_{j=1}^{|V|}\beta^{'}_{i,j}\myscalemat_{j,k}\label{eq:delta_i_k_s_j_k}
\end{align} 

then in matrix form,
\begin{align}\label{eq:ieqbc}
\begin{bmatrix}  
  \delta_{1,1} & \cdots & \delta_{1,|V|} \\
  \vdots       & \ddots & \vdots \\
  \delta_{|V|,1} & \cdots & \delta_{|V|,|V|}
\end{bmatrix}=I={\myscalemat}^{-1}\myscalemat 
\end{align}
thus
\begin{equation}\myscalemat^{-1}=\begin{bmatrix}&\beta^{'}_{1,1}, &\cdots, &\beta^{'}_{1, |V|}\\ 
&\vdots, &\ddots, &\vdots\\ 
&\beta^{'}_{|V|,1}, &\cdots, &\beta^{'}_{|V|, |V|}\end{bmatrix}
\end{equation} 
is the scaled precision matrix.

We also have
\begin{align}
  <X_i^{*}, X_j^{*}>&=<\sum_{j=1}^{|V|}\beta^{'}_{i,j}X_j, \sum_{k=1}^{|V|}\beta^{'}_{i,k}X_k>\\
                    &=\sum_{j=1}^{|V|}\sum_{k=1}^{|V|}\beta^{'}_{i,j}\myscalemat_{j,k}\beta^{'}_{i,k}\\
                    &=\sum_{k=1}^{|V|}\delta_{i,k}\beta^{'}_{i,k}~\Cref{eq:delta_i_k_s_j_k}\\
                    &=\beta^{'}_{i,i}
\end{align}

So~\Cref{eq:partial_corr} can be rewritten as
\begin{align}
\rho_{i,j|\cdot}&=-\frac{<X_{i}^{*},X_{j}^{*}>}{\sqrt{<X_{i}^{*},X_{i}^{*}><X_{j}^{*},X_{j}^{*}>}} \\
                &=-\frac{{[\beta^{'}_{i,j}]}}{\sqrt{{\beta^{'}}_{i,i}{\beta}^{'}_{j,j}}}\\
                &=-\frac{{[\myscalemat^{-1}]}_{i,j}}{\sqrt{{[\myscalemat^{-1}]}_{i,i}{[\myscalemat^{-1}]}_{j,j}}}
\end{align}
\end{proof}

\LabelLemmaInvSigmaSpectrum*% restatement of lemma
\begin{proof}\label{proof:inv_Sigma_spectrum}
The inverse of an invertible congruent transform is still in congruent transform form. 
\begin{align}\label{eq:congruent_inv}
\Sigma^{-1}={[{(I-W)}^{-1}\Omega {(I-W)}^{-T}]}^{-1}={(I-W)}^T\Omega^{-1}(I-W)
\end{align}
The l.h.s.~is with operator ${(I-W)}^{-T}$, the r.h.s.~is with operator $I-W$.

According to Ostrowski's Theorem~\citep{ostrowski1959quantitative} which generalized the Sylvester's law of inertia (congruent transform perserves the number of positive, negative and vanishing eigenvalues), the eigenvalues of the congruent transform is bounded by the congruent transform operator.
\begin{align}
\lambda_k(\Sigma^{-1})=\lambda_k({(I-W)}^T\Omega^{-1}(I-W))&=\theta_k \lambda_k(\Omega^{-1})\\
\lambda_{\min}({(I-W)}^T(I-W)) \le \theta_k &\le \lambda_{\max}({(I-W)}^T(I-W))
\end{align}
The eigenvalues of ${(I-W)}^T(I-W)$ are the squares of the singular values of ${(I-W)}^T$.

Let $v(W, \lambda(W))$ be eigenvector of $W$ corresponding to eigenvalue $\lambda(W)$, then
\begin{align}\label{eq:}
(I-W)v(W,\lambda(W))&=(1-\lambda(W))v(W, \lambda(W)) \\   
{(I-W)}^{-1}v(W, \lambda(W))&=\frac{1}{1-\lambda(W)}v(W, \lambda(W))\\
                            &\text{series decomposintion of}~{(I-W)}^{-1}\nonumber
\end{align}
%for undirected graph: Note $\lambda(W)$ is bounded by the maximum degree of the graph $\rho(W)$.

Since 
\begin{equation}
\det\left(W^{T} - \lambda I\right) = \det\left({(W - \lambda I)}^{T}\right)  = \det (W - \lambda I)\end{equation}

and 
\begin{equation}
\lambda(L)=0
\end{equation} (eigenvalues of a triangular matrix are on the diagonal)

we have
\begin{align}
  \lambda(W)&=\lambda(W^T)=\lambda(PLP^T)=0\\
  Wv(W, \lambda(W))&=PLP^Tv(W, \lambda(W))=0\\
  L\left(P^Tv(W, \lambda(W)\right)&=Lv^{'}(W, \lambda(W))=0\\
  0^T=[Lv^{'}(W, \lambda(W))]^T&=[v^{'}(W, \lambda(W))]^TL^T=0\\ W^Tv(W, \lambda(W))&=PL^TP^Tv(W, \lambda(W))\nonumber\\
                     &=PL^Tv^{'}(W, \lambda(W))\nonumber\\
  &=P([v^{'}(W, \lambda(W))]^TL)^T=0^T=0\label{eq:tranpose_dag_adj_anilate_eigv}
\end{align}

Since
\begin{align}
  {(I-W)}^T(I-W)&=(I-W^T)(I-W)=I-W-W^T+W^TW\\
                &=I-W-W^T+(PL^TP^T)PLP^T\\
                &=I-W-W^T+PL^2P^T\\
                &=I-W-W^T
\end{align}
then
\begin{align}
  {[{(I-W)}^T (I-W)]} v(W, \lambda(W))&=(I-W-W^T)v(W, \lambda(W))\\
                                      &=v(W, \lambda(W))
\end{align}

Are there other eigenvalues of ${(I-W)}^T(I-W)$ other than $1$?

\begin{equation}
\det({(I-W)}^T(I-W)-\lambda I)=\det\left((1-\lambda)I-W-W^T\right)=0
\end{equation}

Since at $\lambda=1$, $\det\left((1-\lambda)I-W-W^T\right)=0=det(W+W^T)$

Let 
\begin{equation}
\lambda^{'}=\lambda-1
\end{equation}

$W+W^T+\lambda^{'}I=\bar{W}+\lambda^{'} I$ is a perturbation to $\bar{W}=W+W^T$ at the diagonal.

$\bar{W}+\lambda^{'} I$ has characteristic polynomial $\det(\bar{W}+\lambda^{'} I-xI)=\det(\bar{W}-(x-\lambda^{'})I)$

$x=0$ is a solution to $\det(\bar{W}-(x-\lambda^{'})I)=0$ where $x$ represent the eigenvalues of $\bar{W}+\lambda^{'} I$.

since 
\begin{equation}
{\det(\bar{W}-(x-\lambda^{'})I)}|_{x=0}= 0 = \det(\bar{W}+\lambda^{'}I)
\end{equation}

$\lambda^{'}$ is an eigenvalue of $\bar{W}$.

so 

\begin{equation}
\lambda^{'}=\lambda(W+W^T)=0=\lambda - 1
\end{equation}

so $\lambda=\lambda(W+W^T)+1$ is also eigenvalue of ${(I-W)}^T(I-W)$.

so 
\begin{align}
  0=\lambda_n({(I-W)}^T(I-W)) \le \theta_k &\le \lambda_1({(I-W)}^T(I-W))\nonumber\\
  &=\lambda_{\max}(W+W^T)+1
\end{align}

i.e.
\begin{equation}
\theta_k \le \rho\left(W+W^T\right) + 1
\end{equation}
\end{proof}

\LabelRestatebleLemmaPrecisionDiagLb*%restatement of lemma
%\begin{lemma}\label{lemma:lb_precision_diag_bt_1_over_sigma_diag_vif}
%The diagonal elements of the precision matrix are lower bounded by the inverse of the corresponding diagonal elements of the covariance matrix, i.e.,
%\begin{equation}
%\Sigma_{p,p}^{-1}\ge \frac{1}{\Sigma_{p,p}}
%\end{equation}
%\end{lemma}
%
\begin{proof}\label{lemma:lower_bound_precision_mat_diag_vif_proof}
\newcommand{\vifSigma}{\Sigma}
Partition the correlation matrix as 
\begin{equation}
\Sigma = \left[\begin{matrix}\Sigma_{1..(p-1)} & \Sigma_{:, p}\\{\Sigma_{:, p}}^T & \Sigma_{pp}
\end{matrix}\right]
\end{equation}
Using block matrix inversion formula, 
\begin{equation}
  {[\Sigma^{-1}]}_{p,p}={(\Sigma_{p,p}-{\Sigma_{:, p}}^T {\Sigma_{1..(p-1)}}^{-1}\Sigma_{:, p})}^{-1}\label{eq:precision_pp}
\end{equation}
Note that~\Cref{eq:precision_pp} is scalar equation, which results in
\begin{align}
  \Sigma_{p,p}-\frac{1}{\Sigma^{-1}_{p,p}}&={\Sigma_{:, p}}^T {\Sigma_{1..(p-1)}}^{-1}\Sigma_{:, p}\label{eq:sigma_pp_minus_vif}
\end{align}

Let the first $p-1$ variables regress on $X_p$, then
\begin{equation}
X_p=[X_1, \ldots, X_{p-1}]\hat{\beta}
\end{equation}

The least-squares estimate of $\beta$ is given by 
%${[X_1, \ldots, X_{p-1}]}^TX_p={[X_1, \ldots, X_{p-1}]}^T{[X_1, \ldots, X_{p-1}]}\beta$.
\begin{align}
\hat{\beta}&={({[X_1, \ldots, X_{p-1}]}^T{[X_1, \ldots, X_{p-1}]})}^{-1}{[X_1, \ldots, X_{p-1}]}^TX_p\\
           &={[\Sigma_{1..(p-1)}]}^{-1}{[X_1, \ldots, X_{p-1}]}^TX_p\\
           &={[\Sigma_{1..(p-1)}]}^{-1}{\Sigma_{:,p}}
\end{align}
corresponding to
\begin{align}\label{eq:}
\hat{X_p}&=[X_1, \ldots, X_{p-1}]{[\Sigma_{1..(p-1)}]}^{-1}{\Sigma_{:,p}}
\end{align}

Let $S$ be any symmetric matrix, $v$ be a matrix, then
\begin{equation}
v^TS^{-1}v = v^T(S^{-1}S)S^{-1}v=v^T(S^{-T}S)S^{-1}v={(S^{-1}v)}^TS(S^{-1}v)\label{eq:vif_plus_minus_trick}
\end{equation},see also~\citep{proof_diag_inv_corr}.
Then take 
\begin{equation}
S=\Sigma_{1..(p-1)}, v=\Sigma_{:,p}
\end{equation}
in~\Cref{eq:vif_plus_minus_trick},~\Cref{eq:sigma_pp_minus_vif} can be rewritten as
\begin{align}
\Sigma_{p,p}-\frac{1}{\Sigma^{-1}_{p,p}}&={\Sigma_{:, p}}^T {\Sigma_{1..(p-1)}}^{-1}\Sigma_{:, p}\label{eq:sigma_pp_minus_vif_2}\nonumber\\
                                      &=
                                      {({\Sigma_{1..(p-1)}}^{-1}{\Sigma_{:,p}})}^T\Sigma_{1..(p-1)}({\Sigma_{1..(p-1)}}^{-1}{\Sigma_{:,p}})\\
                                      &={\hat{\beta}}^T\Sigma_{1..(p-1)}\hat{\beta}\\
                                      &={\hat{\beta}}^T({X_{1..(p-1)}}^TX_{1..(p-1)})\hat{\beta}\\
                                      &=\hat{X_p}^T\hat{X_p}\\
                                      &=Var(\hat{X_p})\\
                                      &=Cov(\hat{\beta},X_p)\\
                                      &\ge 0
\end{align}
so
\begin{equation}
\Sigma_{p,p}\ge \frac{1}{\Sigma^{-1}_{p,p}}
\end{equation}
or 
\begin{equation}
\Sigma_{p,p}^{-1}\ge \frac{1}{\Sigma_{p,p}}
\end{equation}

\end{proof}

\LabelRestatableLemmaUbDiagSigmaOmega*
\begin{proof}\label{proof_lemma:ub_diag_Sigma_omega}
Rewrite~\Cref{eq:corr_y} as
\begin{align*}
\Sigma & :=\mathbb{E}(YY^T) = {(I-W)}^{-1}\Omega {(I-W)}^{-T}
\end{align*}
which is a congruent transformation of $\Omega$ with operator ${(I-W)}^{-T}$.
since 
\begin{align}
{(I-W)}^{-1} = I+W+W^2+\ldots
\end{align}
thus
\begin{align}
  {[{(I-W)}^{-1}]}_{i,i}&=1+W_{i,i} + \sum_k W_{i,k}W_{k,i}+\sum_{k,q}W_{i,k}W_{k,q}W_{q,i}+\cdots\nonumber\\
                        &=1
\end{align}
According to~\Cref{lemma:spectrum_radius_omega_diag_dom}, we have
\begin{align*}
\max_i |\Omega_{i,i}| \le \rho(\Omega)\le 2\max_i |\Omega_{i,i}|%\label{eq:spectrum_radius_diag_dom_diag_ele}
\end{align*}
using Ostrowski's theorem, we have
\begin{align}
\max_i |\Omega_{i,i}| \le \rho(\Sigma)\le 2\max_i |\Omega_{i,i}|
\end{align}
With Gersgorin circle theorem, we have
\begin{align}
\Sigma_{i,i} \le 2\max_i |\Omega_{i,i}|
\end{align}
\end{proof}

\LabelRestatableTheoremParCorr*
\begin{proof}\label{proof_thm:spectrum_par_corr}
According to~\Cref{eq:expression_partial_corr_matrix}, the partial correlation matrix could be written as
\[\tilde{R}=-diag(\sqrt{{[\Sigma^{-1}]}^{-1}_{i,i}})\Sigma^{-1}diag(\sqrt{{[\Sigma^{-1}]}^{-1}_{i,i}})\]
In~\Cref{coro:spectrum_radius_sigma_inv}, we have the spectrum radius of $\Sigma^{-1}$ is $(d_{\max} \max({W}_{i,j})+1)\max_i\frac{1}{\Delta_i(\Omega)}$. 

According to Ostrowski's theorem, the spectrum range of $\tilde{R}$ is restricted by the spectrum range of $\Sigma^{-1}$ with a factor $\theta$ bounded by the spectrum of $diag(\sqrt{{[\Sigma^{-1}]}^{-1}_{i,i}})$:
\begin{equation}
  \min \sqrt{{[\Sigma^{-1}]}^{-1}_{ii}}\le \theta \le \max \sqrt{{[\Sigma^{-1}]}^{-1}_{ii}}\label{eq:theta_bound_ostrowski_precision_mat}
\end{equation}
note 
\begin{equation}
\max \sqrt{{[\Sigma^{-1}]}^{-1}_{ii}} = \frac{1}{\min \sqrt{{[\Sigma^{-1}]}_{ii}}}
\end{equation}
According to~\Cref{lemma:lb_precision_diag_bt_1_over_sigma_diag_vif}
\begin{equation}
{[\Sigma^{-1}]}_{ii}>\frac{1}{\Sigma_{ii}}
\end{equation}
we have
\begin{equation}
\min \sqrt{{[\Sigma^{-1}]}_{ii}}>\min\frac{1}{\sqrt{{[\Sigma]}_{ii}}}
\end{equation}
which implies
\begin{equation}
\frac{1}{\min \sqrt{{[\Sigma^{-1}]}_{ii}}}<\min \sqrt{{[\Sigma]}_{ii}}
\end{equation}
Thus~\Cref{eq:theta_bound_ostrowski_precision_mat} could be rewritten as
\begin{equation}
\min \sqrt{{[\Sigma^{-1}]}^{-1}_{ii}} \le \theta \le \min \sqrt{{[\Sigma]}_{ii}}
\end{equation}
with~\Cref{lemma:ub_diag_Sigma_omega}, we have
\begin{equation}
  \min \sqrt{{[\Sigma^{-1}]}^{-1}_{ii}} \le \theta \le \min \sqrt{{[\Sigma]}_{ii}} \le 2\max_i |\Omega_{i,i}|
\end{equation}
\end{proof}

\end{document}